\documentclass{article}
\usepackage{ijcai24}

\usepackage{times}
\usepackage{soul}
\usepackage{url}
\usepackage[hidelinks]{hyperref}
\usepackage[utf8]{inputenc}
\usepackage[small]{caption}
\usepackage{graphicx}
\usepackage{amsmath}
\usepackage{amsthm}
\usepackage{booktabs}
\usepackage{algorithm}
\usepackage{algorithmic}
\usepackage[switch]{lineno}
\usepackage[table]{xcolor}
\usepackage{appendix}
\definecolor{table_color}{HTML}{FF7F0E}

\usepackage{listings}
\usepackage{mathrsfs}
\usepackage{mathtools}
\usepackage{multirow} 
\usepackage{mathrsfs}
\usepackage{mathtools}
\usepackage{booktabs}
\newcommand\bolden[1]{{\boldmath\bfseries#1}}
\definecolor{codegreen}{rgb}{0,0.6,0}
\definecolor{codegray}{rgb}{0.5,0.5,0.5}
\definecolor{codepurple}{rgb}{0.58,0,0.82}
\definecolor{backcolour}{rgb}{0.95,0.95,0.92}

\lstdefinestyle{mystyle}{
    backgroundcolor=\color{backcolour},   
    commentstyle=\color{codegreen},
    keywordstyle=\color{magenta},
    numberstyle=\tiny\color{codegray},
    stringstyle=\color{codepurple},
    basicstyle=\ttfamily\footnotesize,
    breakatwhitespace=false,         
    breaklines=true,                 
    captionpos=b,                    
    keepspaces=true,                 
    numbers=left,                    
    numbersep=5pt,                  
    showspaces=false,                
    showstringspaces=false,
    showtabs=false,                  
    tabsize=2
}

\title{Interpretable Tensor Fusion}

\author{
Saurabh Varshneya$^1$
\and
Antoine Ledent$^2$\and
Philipp Liznerski$^1$\and
Andriy Balinskyy$^1$\and
Purvanshi Mehta$^3$\and
Waleed Mustafa$^1$\And
Marius Kloft$^1$\\
\affiliations
$^1$RPTU Kaiserslautern-Landau\\
$^2$Singapore Management University\\
$^3$Lica World\\
\emails
\{varshneya, liznerski, balinskyy, mustafa, kloft\}@cs.uni-kl.de, purvanshi@lica.world, aledent@smu.edu.sg
}

\usepackage{amsmath,amsfonts,bm}

\def\Figref#1{Figure~\ref{#1}}

\def\eqref#1{equation~\ref{#1}}

\def\1{\bm{1}}

\DeclareMathAlphabet{\mathsfit}{\encodingdefault}{\sfdefault}{m}{sl}
\SetMathAlphabet{\mathsfit}{bold}{\encodingdefault}{\sfdefault}{bx}{n}

\def\sR{{\mathbb{R}}}

\newcommand{\E}{\mathbb{E}}

\DeclareMathOperator*{\argmax}{arg\,max}
\DeclareMathOperator*{\argmin}{arg\,min}

\newcommand\nbull{{\kern.8pt\raise1.5pt\hbox{\small\bf .}\kern.8pt}}

\DeclareMathOperator{\sumn}{\sum_{i=1}^n}

\usepackage{xcolor}
\usepackage{framed}
\usepackage{comment}
\usepackage{amsmath}
\usepackage{amssymb}
\usepackage{amsfonts}
\usepackage{mathtools}
\usepackage[acronym]{glossaries}
\usepackage{caption}
\usepackage{subcaption}
\usepackage{booktabs} 
\usepackage{multirow}

\DeclareMathOperator{\Fr}{Fr}
\DeclareMathOperator*{\minimize}{minimize}
\includecomment{outline}
\newcommand\blfootnote[1]{%
  \begingroup
  \renewcommand\thefootnote{}\footnote{#1}%
  \addtocounter{footnote}{-1}%
  \endgroup
}
\definecolor{shadecolor}{gray}{0.875}

\newtheorem{theorem}{Theorem}
\newtheorem{thm}{Theorem}[section] \newtheorem{lemma}[thm]{Lemma}
 
\newtheorem{prop}[thm]{Proposition}

\newacronym{mkl}{MKL}{Multiple Kernel Learning}
\newacronym{mnl}{MNL}{Multiple Neural Learning}
\newacronym{vbn}{VBN}{Vector-wise Batch Normalization}
\newacronym{ibn}{IterBN}{Iterative Batch Normalization}
\newacronym{ourfusion}{InTense}{Interpretable Tensor Fusion}
\newacronym{bn}{batch norm}{Batch Normalization}
\newacronym{mfm}{MFM}{Multimodal Factorization Model}
\newacronym{mro}{MRO}{Multimodal Residual Optimization}
\newacronym{toydata}{}{SynthGene}
\newacronym{itf}{}{InTense}
\begin{document}

\maketitle

\begin{abstract}
Conventional machine learning methods are predominantly designed to predict outcomes based on a single data type. However, practical applications may encompass data of diverse types, such as text, images, and audio. We introduce interpretable tensor fusion (InTense), a multimodal learning method for training neural networks to simultaneously learn multimodal data representations and their interpretable fusion. InTense can separately capture both linear combinations and multiplicative interactions of diverse data types, thereby disentangling higher-order interactions from the individual effects of each modality. InTense provides interpretability out of the box by assigning relevance scores to modalities and their associations. The approach is theoretically grounded and yields meaningful relevance scores on multiple synthetic and real-world datasets. Experiments on six real-world datasets show that InTense outperforms existing state-of-the-art multimodal interpretable approaches in terms of accuracy and interpretability. \blfootnote{Code will be available at: https://github.com/srb-cv/InTense}
\end{abstract}
\section{Introduction}
\label{sec:introduction}

The vast majority of machine learning systems are designed to predict outcomes based on a single datatype or ``modality''. However, in various applications---spanning fields from biology and medicine to engineering and multimedia---multiple modalities are frequently in play \cite{he2020advances,lunghi2019multimodal}. The main challenge in multimodal learning is how to effectively fuse these diverse modalities. The most common approach is to combine the modalities in an additive way~\cite{poria2015deep,chagas2020classification}. Such linear combinations suffice in some cases. However, numerous applications necessitate capturing non-linear interactions between modalities. One such instance is sarcasm detection, described in Figure \ref{fig:concept_fig}~\cite{hessel2020does}. The arguably most popular approach to capture non-linear interactions of modalities is ``Tensor fusion''~\cite{zadeh2017tensor,tsai2019multimodal,liang2021multibench}. The main idea in tensor fusion is to concatenate modalities via tensor products in a neural network. 

\begin{figure}[t] 
	\centering
        \includegraphics[width=\columnwidth]{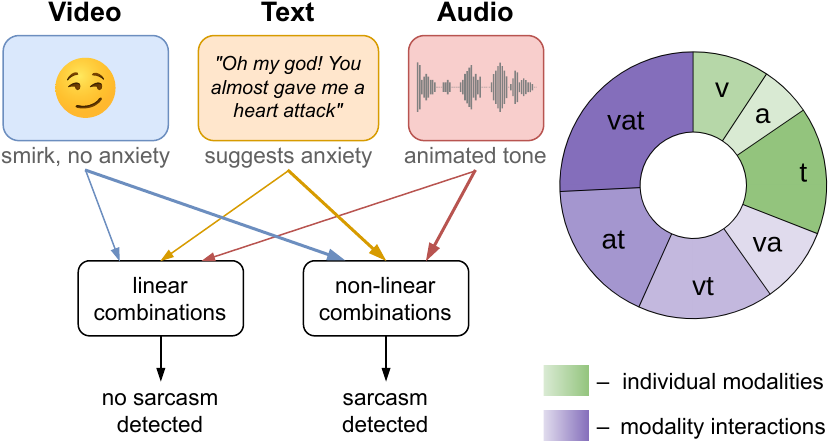}
     \caption{
    Left is an excerpt of the MUStARD dataset on sarcasm detection, where the proposed InTense method sets a new state-of-the-art. (See Section 4 for details.) A linear combination of modalities fails here because the expressions of happiness and anxiety combine to something neutral rather than sarcasm. To detect sarcasm, the interactions among modalities are crucial. InTense captures these interactions and assigns them with interpretable relevance scores, shown in the pie chart. Scores for individual modalities and their interactions are colored green and blue, respectively. InTense reveals that interactions are crucial for successful sarcasm detection.
     } 
     \label{fig:concept_fig}
\end{figure}

A substantial drawback of tensor fusion is its inherent lack of interpretability, which can significantly hinder its application in real-world scenarios. Interpretable multimodal models may reveal the relative importance of modalities \cite{buchel1998multimodal,hessel2020does}, unveiling spurious modalities and social biases in the data. Identifying interactions among modalities is the main goal in several application domains. For instance, in statistical genetics, it is crucial to identify the interactions among Single Nucleotide Polymorphisms (SNPs) that contribute to the inheritance of a disease~\cite{behravan2018machine,elgart2022non}. Although some interpretable multimodal methods exist, they are limited to linear combinations or require resource-intensive post hoc algorithms for interpretation.

In this paper, we introduce \emph{interpretable tensor fusion} (InTense), which jointly learns multimodal neural representations and their interpretable fusion. InTense provides out-of-the-box interpretability by assigning relevance scores to all modalities and their interactions. Our approach is inspired by multiple kernel learning~\cite{kloft2011lp}, a classic kernel-based approach to interpretable multimodal learning, which we generalize to deep neural networks and term \emph{Multiple Neural Learning} (MNL). While both MNL and InTense provide relevance scores for the modalities, InTense additionally produces scores for the interactions of modalities. These interaction scores are made possible through a novel interpretation of neural network weight matrices: We show that neural networks tend to favor higher-order tensor products, leading to spurious interpretations (i.e., overstating high-order interactions between modalities). We resolve this issue by deriving a theoretically well-founded normalization approach. In the theoretical analysis, we prove that this produces genuine relevance scores, avoiding spurious interpretations. In extensive experiments, we empirically validate the relevance scores on data and show that InTense outperforms existing state-of-the-art multimodal interpretable approaches in terms of accuracy and interpretability.

In summary, our contributions are:
\begin{itemize}
\item We introduce \emph{Multiple Neural Learning} (MNL), a theoretically guided adaptation of the established Multiple Kernel Learning algorithm to deep learning.
\item We introduce InTense, an extension of MNL and tensor fusion designed to capture non-linear interactions among modalities in an interpretable manner.
\item We provide a rigorous theoretical analysis that provides evidence of the correct disentanglement within our fusion framework.
\item We validate our approach through extensive experiments, where we meet the state-of-the-art classification accuracy while providing robust interpretability.
\end{itemize}

\section{Related Work}
\label{sec:related_work}
We now review existing multimodal learning methods that produce interpretability scores for the modalities.

\paragraph{Interpretable Methods for Learning \emph{Linear} Combinations of Modalities.}

The vast majority of interpretable multimodal learning methods consider linear combinations of modalities. The arguably most popular instance is \acrfull{mkl}, where kernels from different modalities are combined linearly. Here, a weight is learned for each kernel determining its importance in the resulting linear combination of kernels~\cite{kloft2011lp,rakotomamonjy2008simplemkl}. However, the performance of MKL is limited by the quality of the kernels. Finding adequate kernels can be especially problematic for structured high-dimensional data, such as text or images. Addressing this, several authors have studied combining multiple modalities using neural networks in a linear manner~\cite{poria2015deep,chen2014emotion,arabaci2021multi}. However, these representations are independently learned to form basis kernels and later combined in a second step through an SVM or another shallow learning method. Such independently learned representations cannot properly capture modality interactions.

\paragraph{Methods for Learning \emph{Non-linear} Combinations of Modalities.} 
\citeauthor{hessel2020does}~\shortcite{hessel2020does} map neural representations to a space defined by a linear combination of the modalities. While they quantify the overall importance of non-linear interactions, they do not provide scores for individual modality interactions. \citeauthor{tsai2020multimodal}~\shortcite{tsai2020multimodal} introduce multimodal routing, which is based on dynamic routing~\cite{sabour2017dynamic}, to calculate scores for the modality interactions. These scores depend on the similarity of a modality's representation to so-called concept vectors, where one such vector is defined for each label. However, routing does not distinguish between linear and non-linear combinations and is thus misled by partially redundant information in the combinations. Indeed, we show through experiments (see Section 4) that the non-linear combinations learned by routing are incorrectly overestimated.
\citeauthor{gat2021perceptual}~\shortcite{gat2021perceptual} propose a method to obtain modality relevances by computing differences of accuracies on a test set and a permuted test set. However, this method has limited interpretability and requires multiple forward passes through the trained network to obtain relevance scores.
\citeauthor{wortwein2022beyond}~\shortcite{wortwein2022beyond} learn an aggregated representation for unimodal, bimodal, and trimodal interactions, respectively. However, their method does not learn fine-grained relevance scores for the various combinations of modalities. Alongside methods offering limited interpretability, there exist methods that non-linearly combine modalities without adding any interpretability~\cite{zhang2023meta,liang2021multibench,tan2019lxmert}.

In summary, none of these methods learns proper relevance scores of interactions between modalities.

\paragraph{Post-hoc Explanation Methods.}
There exist several methods for post-hoc explanation of multimodal learning methods~\cite{gat2021perceptual,chandrasekaran2018explanations,park2018multimodal,kanehira2019multimodal,cao2020behind,frank2021vision}. These methods consist of two steps: first, training a multimodal model that is not inherently interpretable, followed by the calculation of relevance scores in hindsight. However, their two-step nature makes these methods challenging to analyze theoretically. Moreover, since the initial model disregards interpretability, it may lead to inherent limitations in the explanatory process. Additionally, these methods come with the added computational burden of producing relevance scores. Another limitation is their applicability, which is confined to specific types of modalities. 




\section{Methodology}
In the following sections, we introduce several components comprising our approach.
First, we review the classical $L_p$-norm \acrfull{mkl} framework~\cite{kloft2011lp}, which we extend to Multiple Neural Learning. Subsequently, we propose Interpretable Tensor Fusion (InTense), which captures non-linear modality interactions. Furthermore, we show how InTense learns disentangled neural representations, thereby computing correct relevance scores. 
\subsection{Preliminaries}
We consider a dataset $\{(x_i,y_i)\}_{i=1}^n$ with labels $y_i \in \{-1, 1\}$. The inputs have $M$ modalities, where $x_i^m \in \mathcal{X}^m$ for $m\in \{1, \ldots, M\}$ denotes the $m^{th}$ modality of the datapoint $x_i$, and $\mathcal{X}^m$ is the input space associated with the modality $m$. 
In \acrshort{mkl}, one considers kernel mixtures of the form
$k(u,v)=\sum_{m=1}^M\beta_mk_m(u,v),$  
where $k_m(u,v)$ is a base kernel and $\beta_m\geq0$ for all $m$. Imposing an $L_p$-norm constraint on the vector $\beta\in \mathbb{R}^M$ gives rise to the following classic optimization problem: 

\begin{multline}
\label{MKLL}
\minimize_{\substack{w_1,w_2,\cdots,w_L, \beta \\ \beta\in \sR^M,\beta\geq 0, \|\beta\|_p\leq 1}} \Bigg(\sumn \ell\Big(  \sum_{m=1}^M \sqrt{\beta_m} \langle w_m, \Psi_m(x_i^m)\rangle_{\mathcal{H}^m}\\ +b ,y_i\Big)
+\frac{\lambda}{2}\sum_{m=1}^M \|w_m\|_{L^2(\mathcal{H}^m)}^2\Bigg), \;  
\end{multline} 

\noindent where $\ell$ is a loss function, and $\Psi_m:\mathcal{X}^m\rightarrow \mathcal{H}^m$ are feature maps from the input space $\mathcal{X}^m$ to the Hilbert space $\mathcal{H}^m$ associated with kernel $k_m$ such that for each $m\in \{1,2,\ldots,M\}$ and $u,v\in \mathcal{X}^m$, $k_m(u,v)=\langle \Psi_m(u),\Psi_m(v)\rangle_{\mathcal{H}^m}$, where $\langle \cdot,\cdot \rangle_{\mathcal{H}}^m$ denotes the inner-product associated with the Hilbert space $\mathcal{H}^m$. The base kernels $k_m$ are assumed to be fixed functions. This is in sharp contrast to \acrfull{mnl} introduced in the next section, where each feature map $\Psi_m(x)$ is learned from the data. 
\subsection{Multiple Neural Learning\label{sec:deepmkl}}
In this section, we propose Multiple Neural Learning (MNL), \textbf{an interpretable method for linear combination of modalities}.
In MNL, we train a neural network composed of two components: 1) modality subnetworks that output a neural representation for each modality and 2) a linear fusion layer that combines the representations in an interpretable manner. 
We define the optimization problem as:
\begin{multline}
\label{eq:BASIC}
\minimize_{\substack{w_L^1, \cdots ,w_L^M, \beta, \\ W^1,W^2, \cdots, W^M,\\ \beta\in \sR^M,\beta\geq 0, \|\beta\|_p\leq 1}}\Bigg(\sum_{i=1}^n \ell\Big(\sum_{m=1}^M \sqrt{\beta_m} \langle w_L^m,f^{m}(x_i^m)\rangle\\[-20pt] +b,y_i\Big)
+\Lambda \sum_{l=1}^L\sum_{m=1}^M \|w^m_l\|_2^2\Bigg),
\end{multline}
\noindent where $f^m~$
is the $m^{th}$ modality's subnetwork composed of $L-1$ layers with weights $W^{m} = \{w^{m}_1,\ldots,w^{m}_{L-1}\}$.
A representation for the $i^{th}$ data point's $m^{th}$ modality is obtained by $f^m(x_i^m)$. The fusion layer $L$ with weights $w_L^{1},\dots,w_L^{M}$ learns a linear combination of the modality representations.
$\Lambda$ and $p\;  
(1\leq p<\infty)$  are hyperparameters and $\ell(t,y)= -\log\left(\frac{\exp(ty)}{1+\exp(ty)}\right)$ is the cross-entropy loss function. 
This setup can also be seen as additive fusion because it represents a linear combination of the modalities with weights $\sqrt{\beta_m}$.

Notably, $\sqrt{\beta_m}$ is a positive weight for the $m$'th modality, indicating its relevance score. The vector $\beta$ is simultaneously optimized with the network weights. However, the constraints on $\beta$ introduce an increased difficulty in optimizing \eqref{eq:BASIC}. The following theorem presents a simplified optimization problem by eliminating $\beta$ from \eqref{eq:BASIC} along with a method to retrieve $\beta$ from the learned weights.

\begin{theorem}
    \label{DasTheorem}
    The optimization problem in~\eqref{eq:BASIC} is equivalent to the following problem, where the parameters $\beta$ are no longer present:
    
\begin{align}
\label{eq:deepmkl}
\minimize_{\substack{w_L^1,w_L^2,\ldots,w_L^M, \\ W^1,W^2, \cdots, W^M}}\sum_{i=1}^n \ell\left(\sum_{m=1}^M  \langle w_L^m,f^{m}(x^m_i)\rangle +b,y_i\right) \nonumber \\[-5pt]
+ \Lambda \sum_{l=1}^{L-1}\sum_{m=1}^M \|w^m_l\|_2^2+\Lambda\left(\sum_{m=1}^M \|w^m_L\|_2^q\right)^{\frac{2}{q}},
\end{align}

\noindent where $q=\frac{2p}{p+1}$ (and therefore $1\leq q\leq 2$).
The corresponding values of relevance score $\beta$ can be recovered after the optimization as: 
            
\begin{align}
             \label{eq:recoverBeta}
\beta_{m}= \frac{     \|w^m_L\|_2   ^{\frac{2}{p+1}}       }{    \left(\sum_{\tilde{m}=1}^M \|w^{\tilde{m}}_L\|_2^{\frac{2p}{p+1}}\right)^{\frac{1}{p}}         }.
    \end{align}	
\end{theorem} 

The theorem states that the relevance of a modality in our jointly trained network can be obtained by applying a suitable $p-$norm over the weights of the fusion layer $L$. A detailed proof of the theorem can be found in Appendix A. The central idea is observing that the parameters $\beta$ can be absorbed into the weights $w_L^m$, pushing $\beta$ into the regularization term. Subsequently, by showing that $\beta$ can be minimized independently from the weights, the optimal value is attained through~\eqref{eq:recoverBeta}. Absorbing $\beta$ in fusion weights in turn introduces the additional block $L^q$ norm regularization term $\Lambda(\sum_{m=1}^M \|w^m_L\|_2^q)^{\frac{2}{q}}$. 

\paragraph{Correct Relevance Scores Through Normalization.}
In pre-experiments (see Appendix F.1) we found that the relevance scores can be misleading, especially when the network outputs high activation values for some modalities. We address this issue with proper normalization techniques.
We propose an adaptation of the standard \acrfull{bn}~\cite{ioffe2015batch}, which we call \acrfull{vbn}.
VBN ensures that the $L_2$-norm of the activation values, \textit{averaged over a mini-batch}, is constant.
Let $B$ be a mini-batch of datapoints' indices. We define VBN as:  

\begin{align}
\label{eq:vbn_eq}
\tilde{f}^m(x_i^m)=\frac{f^m(x^m_i)-\mu_{B,m}}
{\sigma_{B,m}},  \text{where} \\
\mu_{B,m}=\frac{\sum_{i\in B }f^m(x^m_i)}{|B|}, \text{and}\\ 
\label{eq:vbn}
\sigma_{B,m}^2=\frac{\sum_{i\in B }\|f^m(x^m_i)-\mu_{B,m}\|_2^2}{|B|}.    
\end{align}

The mean $\mu_{B,m}$ is computed element-wise as in the standard batch norm. However, in equation \ref{eq:vbn}, instead of computing the variance element-wise, we calculate the average of the squared $L_2-$norm of a modality representation across the mini-batch.
Unlike batch norm, we do not shift and scale the representations element-wise after the normalization step.
Using \acrshort{vbn}, the loss in \eqref{eq:deepmkl} changes to:  

$$\sum_{i=1}^n \ell\left(\sum_{m=1}^M  \left\langle w_L^m,\tilde{f}^{m}(x^m_i)\right \rangle +b,y_i\right).$$  
Note that \acrshort{vbn} is applied \textit{after} the activation function to obtain $\tilde{f}^m(x^m)$.
We found empirically that proper normalization is crucial for MNL to achieve competitive performance.


\subsection{Interpretable Tensor Fusion}  \label{sec:intense}
In this section, we propose Interpretable Tensor Fusion (InTense), an extension of MNL that additionally produces scores for interactions (non-linear combinations) of modalities. InTense is based on tensor fusion, which captures multiplicative interactions among modalities by computing a tensor product over the individual modality representations~\cite{zadeh2017tensor}. InTense operates as follows.
For a dataset with $M$ modalities, we consider interactions up to a given order of $D$, where $D \leq M$.
An order of $D$ implies interaction among $D$ modalities.
A multiplicative interaction of modalities is defined by a subset $I \in \mathcal{I}$, where $\mathcal{I} = \{ J \subset \{1,\dots,M\} : |J| \leq D\}$, and a tensor product $f^I(x):=f^{I_1}\otimes f^{I_2}\otimes \ldots f^{I_{|I|}}$, where $f^{I_m}$ is the representation of modality $I_m$, and $\otimes$ denotes the tensor product operator. Analogously to \eqref{eq:BASIC}, we obtain a new objective:

\begin{multline} \label{eq:tensor_fusion_basic}
	\minimize_{\substack{w_L^I, \beta_I: I \in \mathcal{I},\\\|\beta\|_p\leq 1, \beta_I \ge 0   \\ W^1,W^2, \cdots, W^M}}\Bigg(\sum_{i=1}^n \ell\Big(\sum_{I\subset \mathcal{I} } \sqrt{\beta_I} \langle w_L^I,f^{I}(x)\rangle 
	+b,y_i\Big)\\[-20pt]+\Lambda \sum_{l=1}^L\sum_{m=1}^M \|w^m_l\|_2^2\Bigg)
\end{multline}
This optimization problem can be seen as a special case of \acrshort{mnl}, where the multiplicative interactions are treated as separate modalities. 
Therefore, in combination with \eqref{eq:tensor_fusion_basic}, Theorem \ref{DasTheorem} computes the relevance scores for all modalities and their interactions.

\subsubsection{What Can Go Wrong?}
In our experiments with synthetic multimodal datasets (Section \ref{sec:syn_data}), we found that relevance scores of higher-order interactions are greatly overestimated. Scores can be high even when no true interactions exist in the data. We call this phenomenon \textit{higher-order interaction bias}. The bias is caused by higher-order tensor products corresponding to very large function classes, which approximately include the function classes corresponding to lower-order tensors as subsets. 

Indeed, it is possible that a linear combination of the components of a tensor product learns the same functions as a linear combination of the individual-modality representations. For instance, consider two modalities $(m_u, m_v)$ and their representation vectors as $\mathbf{u}, \mathbf{v} \in \mathbb{R}^3$. Assume the first component of the learned representations is constant (e.g., 1), i.e.,  $\mathbf{u} = [1, u_2,u_3]^{\intercal}$ and $\mathbf{v} = [1, v_2, v_3]^{\intercal}$. In such a scenario, the linear combination $\alpha_1 u_2 + \alpha_2 v_2$ (for $\alpha_1, \alpha_2 \in \mathbb{R}$) can also be represented as $\alpha_1 (\mathbf{u} \otimes \mathbf{v})_{2,1} + \alpha_2 (\mathbf{u} \otimes \mathbf{v})_{1,2}.$ 
Using the \acrshort{mnl} algorithm, the relevance scores for modalities $m_u$ and $m_v$ are $\alpha_1$ and $\alpha_2$ respectively. 
However, the relevance score for the modality with tensor product ($m_{u \otimes v}$) is $\sqrt{(\alpha_1)^2+(\alpha_2)^2}$. Here, the $L^p$-norm regularization with any $p<2$ will favor the representation $m_{u \otimes v}$.
Therefore, if the dimensions of the modality representations $\mathbf{u}$ and $\mathbf{v}$ are strictly greater than required to represent the ground truth (which is usually the case in modern networks), lower-order functions will be preferably represented inside the higher-order products by learning a constant in the representations of each modality. 
Our experiments show that a trained network typically exhibits such behavior.
We propose a solution to this problem of higher-order interaction bias in the rest of this section.



\paragraph{Correct Bi-modal Interactions.}
We now address the problem of higher-order interaction bias. The key idea is to introduce a normalization scheme that downweights higher-order interactions. Our normalization scheme is a sophisticated generalization of the \acrfull{vbn} scheme described in equation~\ref{eq:vbn_eq}. Let $m_1$ and $m_2$ be two modalities and define their representations as $f^{m_1}$ and $f^{m_2}$. The representation of the bi-modal interaction is defined as $f^{\{m_1, m_2\}} = f^{m_1} \otimes f^{m_2}$. In this simple bi-modal case, our solution can be summarized as follows: we apply \acrshort{vbn} to $f^{m_1}$ and $f^{m_2}$ before taking the product, and finally apply \acrshort{vbn} again to the result. Formally, normalize each modality representation according to equation \ref{eq:vbn_eq} as:

$$\tilde{f}^{m}(x_i)=\frac{ f^m(x_i)-\mathbb{E}(f^m(x_i))   }     
{\sqrt{ \mathbb{E}\left( \|  f^m(x_i)-\mathbb{E}(f^m(x_i))       \|^2_2   \right) } } ,$$
then compute the tensor-product as:
$$\hat{f}^{\{m_1,m_2\}}(x_i)=\tilde{f}^{m_1}\otimes \tilde{f}^{m_2},$$
and similarly apply VBN to the tensor-product to obtain $\tilde{f}^{\{m_1,m_2\}}$.

The centering step of normalization is applied element-wise over a mini-batch. Thus, if a few components were to be non-zero constants in a mini-batch, they would become zero after the normalization. This ensures that $\tilde{f}^{m_1,m_2}$ cannot easily access lower degree information contained in $\tilde{f}^{m_1}$, because elements in $\tilde{f}^{m_2}$ cannot be a non-zero constant, and vice versa. The normalization could seemingly be trivially extended to more than two modalities by applying the normalization iteratively up to an n-order tensor product.
However, such an extension may still lead to a high-interaction bias. We illustrate why such a trivial extension may not work for more than two modalities and later generalize the normalization scheme for any number of modalities.

\paragraph{Generalization Over M-modal Interactions.}
Extending the aforementioned normalization to cases where $D > 2$ is not straightforward. This complexity arises because, when fusing more than two modalities, potentially, the representations of a subset of $M$ modalities conspire to produce a constant even though each individual modality representation is non-constant. For instance, consider three modalities with one-dimensional representations and apply \acrshort{vbn} to the representations  $f^{1}$, $f^{2}$, $f^{3}$, then to $f^{1,2}$, and finally to  $f^{1,2,3}=\tilde{f}^{1,2}\otimes \tilde{f}^{3}$, it is  still possible that representation $f^{1}$ is learned in the higher-order tensor product.
For instance, assume the components of $f^2$ and $f^3$ are learned to satisfy the following for each datapoint: 
(1) $f^{2}=f^{3}$;
(2) we have that $f^{2}$ is a Rademacher variable ($P[f^2=1]=P[f^2=-1]=0.5$);
and (3) $f^{2}$ is independent of $f^{1}$. 
Since $f^{1,2,3} = f^1f^2f^3$ and $f^{2}f^{3}=1$ for all datapoints, we actually have $f^{1,2,3}=f^{1}$, and this seemingly higher order combination can still recover the first modality. 

We address this issue by carefully normalizing the modalities features.
The key is to prevent the combination of features of one or more modalities from resulting in a constant value.
Similar to the bi-modal scenario, we need to ensure that the contribution of a subset of modalities to the larger fusion set is, on average, zero.
This guarantees that no constant value, other than zero, is multiplied by the product of the complement of that subset within the original fusion set.
Formally, for each $I\subset \{1,2,\ldots,M\}$, the centering step of our batch norm procedure is defined as follows, where we first assume each modality is one-dimensional to simplify the exposition:

\begin{align}
\label{eq:centering}
	\hat{f}^{I} &= \sum_{\ell=0}^{|I|}(-1)^{\ell} \sum_{\emptyset \neq S^1,\ldots,S^{\ell}\subset I \atop  S^1,\ldots,S^{\ell} \text{disjoint}} \prod_{m \in I \setminus (\cup S^{k})}f^m \nonumber \\ &\; \prod_{k\in \{ 1,2,\ldots, \ell\}}  \mathbb{E}\left(\prod_{m\in S^k}f^{m}\right).
\end{align}

When the modalities are multi-dimensional, the above operation is applied independently to each multi-index component~\footnote{In particular, the multivariate case could be expressed with a similar formula as equation~\ref{eq:centering} with the products replaced by outer tensor products, but this would require a different reordering of the components for each term of the sum.}. 
After preforming the centering step above for each multi-index component, we perform the generalized normalization step as follows

\begin{align}\tilde{f}^I=\frac{\hat{f}^{I}}{\sqrt{\E\|\hat{f}^{I}\|_{\Fr}^2     }   }.
\end{align}

While the solution can no longer be easily interpreted as a composition of standard batch norm operations, it is, in fact, possible to show that lower-order fusion can not be represented by a linear combination of their higher-order counterparts. 
Theorem \ref{original} formalizes this result. 
\begin{theorem}
	\label{original}
	The centering step described in \eqref{eq:centering} can be represented as the multi variate polynomial:  
	\begin{align*}
	\sum_{J\subset I}\mathcal{G}_{J} \prod_{m\in J} f^m ,
 \end{align*}
 for some real coefficients $\mathcal{G}_J$. Furthermore, the expected contribution of a subset of modalities $J$ in the fusion of the set of modalities $I$, where $J\subsetneq I$ is zero. That is, we have for any $J\subsetneq I  $ (including the empty set), 

	\begin{align}
 \label{eq:maintheorem}
		\mathbb{E}\left(    \sum_{K: J\subset K \subset I} \mathcal{G}_{K}  \prod_{m\in K\setminus J} f^{m}             \right)=0.
	\end{align}
 
\end{theorem}

The theorem states that the expected value of the contribution of any subset $J \subsetneq I$ of modalities is zero in the fusion of $I$. Thus, $\tilde{f}^I$ (higher-order) can not learn a linear combination of the $\tilde{f}^J$ (lower-order).
Appendix B contains a comprehensive proof of the theorem.

To make the exposition clearer, we provide as an example the case $I=\{1,2,3\}$ and the individual representations $f^m$ are standardized using \acrshort{vbn}. In this case, we have, using the notation $\overline{f^1f^2}=\mathbb{E}(f^1f^2)$: 

\begin{align*}
    \hat{f}^{1,2,3}&=f^1\times f^2\times f^3-\overline{f^1\times f^2}\times f^3 -\overline{f^1\times f^3}\times f^2\\
&-\overline{f^2\times f^3}\times f^1-\overline{f^1\times f^2\times f^3}.
\end{align*}
An elaborated centering step, without normalizing the individual modality representations, and strictly following equation~\ref{eq:centering} is described in Appendix B.1.


In this section, we introduced \emph{\acrfull{ibn}}, a normalization scheme addressing higher-order interaction bias in multimodal learning. 
In the next section, we show the effectiveness of our method on synthetic and real-world datasets.

\section{Experiments}
First, we experiment on synthetic data, where we control the amount of relevant information in the modalities, and compare InTense's relevance scores to the established ground truth. 
Second, we compare the predictive performance of InTense with popular multimodal fusion methods on six real-world multimodal datasets.

\begin{figure}[b]
	\centering
	\includegraphics[width=0.8\columnwidth]{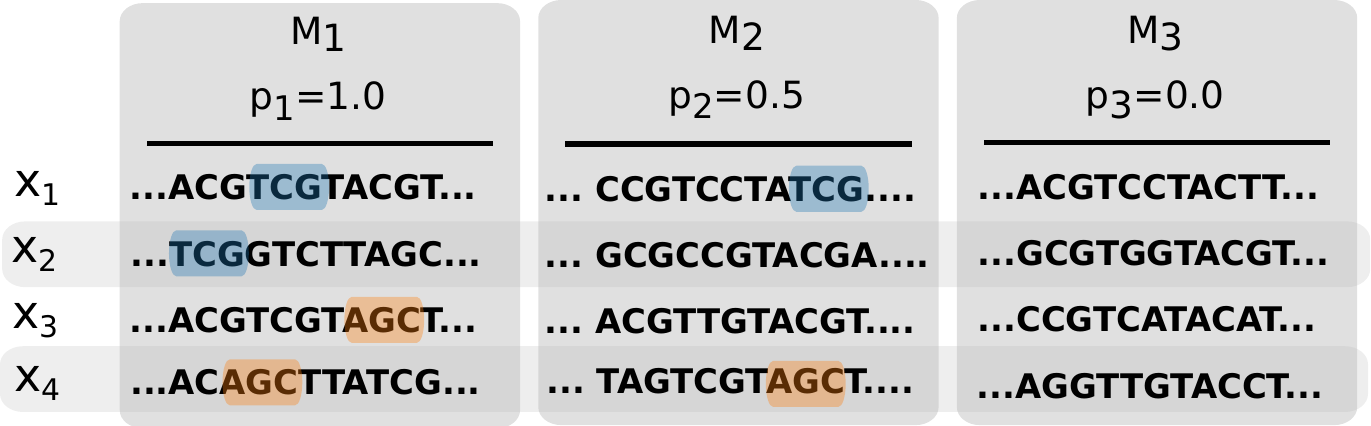}
	\caption{An excerpt of three modalities of \textsc{\acrlong{toydata}}, our self-curated binary classification dataset, where each sequence is made from a set of letters \{A,C,G,T\}. A positive class-sequence ``TCG" and a negative class-sequence ``AGC" is added according to the probability $p_m$.}
	\label{fig:data_glimpse}
\end{figure}
\subsection{Evaluating the Relevance Scores}
\label{sec:syn_data}
We created a multimodal dataset where each modality of a datapoint is a sequence of letters chosen randomly from a predefined set. 
For each datapoint $x$ and modality $m$, an informative subsequence is inserted at a random position with a probability of $p_m$.
We call our dataset \textsc{\acrlong{toydata}}. 
More details about it can be found in Appendix C.

We perform two experiments to determine the correctness of the relevance scores obtained from InTense. 
First, we construct a binary classification dataset with labels that ensure the modalities are independent and do not interact. 
Second, we generate another set of labels that can only be predicted using non-linear interactions among the modalities.

\subsubsection{InTense Assigns Correct Relevance Scores to \emph{Independent} Modalities}
In this set of experiments, we create a synthetic dataset with independent modalities (i.e., without interactions). As a baseline, we train one model for each modality and then compare the accuracies of those models with the relevance scores of InTense trained on all modalities together.

 \paragraph{Dataset.} We create the \textsc{\acrlong{toydata}} dataset, where for each modality $m$, each datapoint with a positive/negative class label contains a class-specific sequence with a probability $p_m$ independent of other modalities. A high value of $p_m$ indicates that most sequences in the modality $m$ contain a class-specific sequence. Thus, the higher the value of $p_m$, the more relevant the modality $m$ becomes. The labels for all datapoints are uniformly distributed between the two classes. \Figref{fig:data_glimpse} shows how probability $p_m$ affects modality relevance. We use $10$ modalities where the informative subsequence is inserted into modalities $M_2$, $M_4$, and $M_7$. There is no discriminative information present in other modalities. 

\paragraph{Results.} \Figref{fig:syn_data_mkl} shows the relevance scores calculated by InTense on \acrlong{toydata}. InTense assigns the correct relevance scores as they align with human intuition. The higher the probability $p_m$, the more informative signal is contained in modality $m$, and the higher the predicted relevance score. We further validate the correctness of InTense's interpretability by comparing it with the accuracies obtained from unimodal models trained on each modality separately. Again, InTense's relevance scores correlate with the unimodal accuracies.

\begin{figure}
	\centering
	\includegraphics[width=0.9\columnwidth]{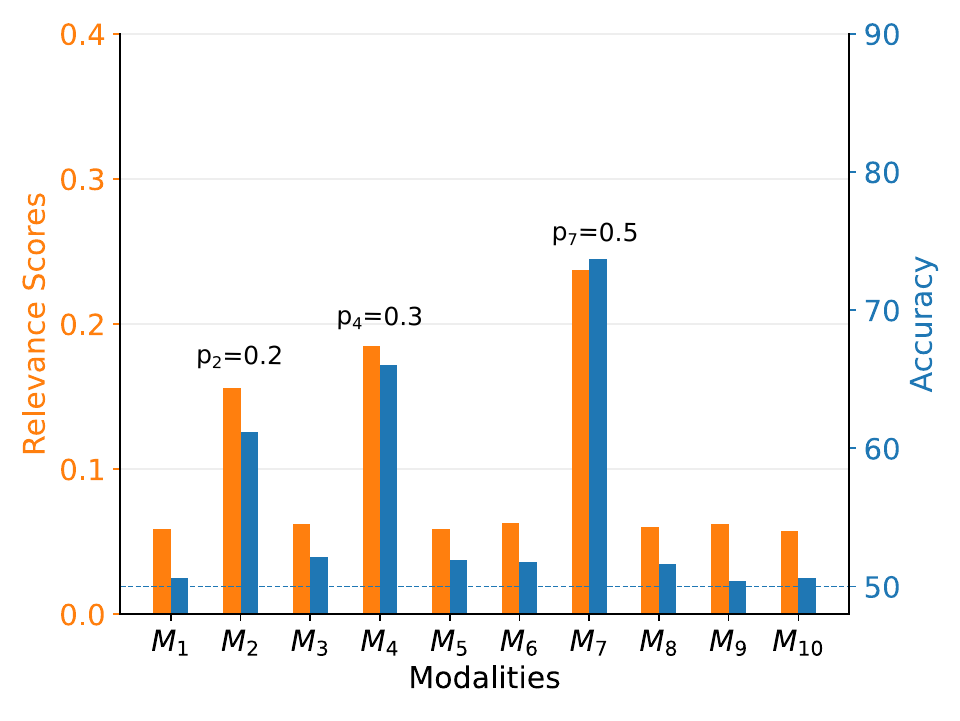}
	\caption{The figure shows a high correlation of InTense' relevance scores and accuracies of unimodal models on \acrlong{toydata}. The modalities $M_2, M_4, M_7$ achieve high relevance scores and high accuracy as they contain class-specific information. Other modalities contain no class-specific information, which leads to a very low relevance score and an accuracy of around $50\%$  (equivalent to random guessing).}
	\label{fig:syn_data_mkl}
\end{figure}

\subsubsection{InTense Assigns Correct Relevance Scores to \emph{Interacting} Modalities} \label{sec:interaction_scores}
We now turn to a situation where the label depends on a non-linear interaction among the modalities by design. 

\paragraph{Dataset.} We also create the \textsc{\acrlong{toydata}-tri} dataset, a trimodal version of \textsc{\acrlong{toydata}}. However, this time, the informative subsequence is not class-specific. Instead, the label is defined by an exclusive-or (XOR) relationship between the first two modalities ($M_1$ and $M_2$). The label is $0$ if both modalities contain the subsequence or none of them does, and the label is $1$ otherwise (i.e., when one of the modalities contains the subsequence).
Note that modality $M_3$ does not contain any informative subsequence; thus, it is irrelevant.
As before, we generate a balanced dataset with 50\% of the samples being positive and 50\% negative. 


\paragraph{Results.} 
\Figref{fig:mkl_iter_bn} shows the results. 
We observe that the global relevance scores calculated by the MultiRoute baseline~\cite{tsai2020multimodal} are biased toward high-order interactions. This occurs even when the interactions do not add any useful information and thus should have been discarded by MultiRoute. We further see that the proposed InTense method does avoid this bias and correctly assigns a high relevance score solely to the interaction of $M_1$ and $M_2$. This shows that InTense can ensure the correctness of relevance scores even when higher-order interactions are involved.


\begin{figure}
	\centering
	\includegraphics[width=0.85\columnwidth]{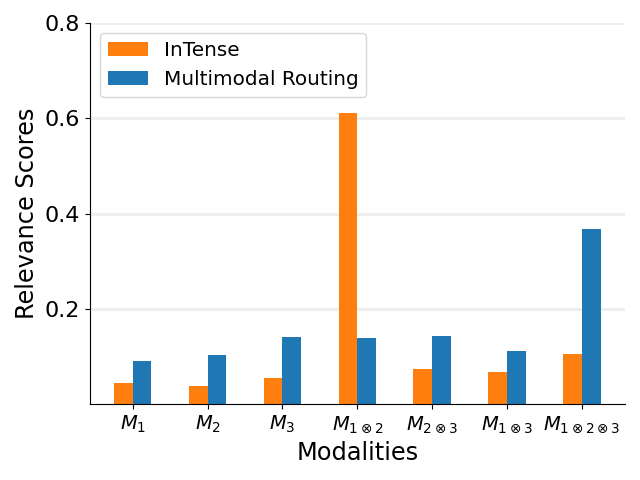}
	\caption{Illustration of the relevance scores calculated by the proposed InTense and the MultiRoute baseline when higher-order modality interactions are involved in the ground truth. MultiRoute leads to biased results (blue bars), where the relevance scores are concentrated toward higher-order interactions $M_{1\otimes2\otimes3}$. In contrast, InTense (orange bars) correctly assigns a high relevance score only to the interaction $M_{1\otimes2}$, which contains all class-specific signals.}
	\label{fig:mkl_iter_bn}
\end{figure}

\subsection{InTense Performs SOTA in Real-World Applications}

We demonstrate the effectiveness of InTense in providing interpretability without compromising predictive performance across a range of real-world applications. 
In order to compare performance and ensure reproducibility, we followed the experimental setup (e.g., data preprocessing, encodings of different modalities) of the MultiBench~\cite{liang2021multibench} benchmark for all the experiments. 

\paragraph{Sentiment analysis.}
In sentiment analysis, also known as opinion mining, the target is to identify the emotional tone or feeling underlying the data. Initially confined to text data, sentiment analysis has evolved to encompass multiple modalities. The task becomes challenging due to the intricate interactions of modalities, which play a significant role in expressing sentiments. Understanding sentiments is crucial in business intelligence, customer feedback analysis, and social media monitoring. To evaluate InTense in sentiment analysis, we employed CMU-MOSEI~\cite{bager-zadeh-etal-2018-multimodal}, the largest dataset of sentence-level sentiment analysis for real-world online videos, and CMU-MOSI~\cite{zadeh2016mosi}, a collection of annotated opinion video clips.

\paragraph{Humor and Sarcasm detection.}
Humor detectors identify elements that evoke amusement or comedy, while sarcasm detection aims to discern whether a sentence is presented in a sarcastic or sincere manner. Sarcasm and humor are often situational. Successfully detecting them requires a comprehensive understanding of various information sources, encompassing the utterance, contextual intricacies of the conversation, and background of the involved entities. As this information extends beyond textual cues, the challenge lies in learning the complex interactions among the available modalities. To assess our approach's effectiveness in these tasks, we utilized UR-FUNNY~\cite{hasan2019ur} for humor detection and MUStARD~\cite{castro2019multimodal} for sarcasm detection.

\paragraph{Layout Design Categorization.}
Layout design categorization is about classifying graphical user interfaces into predefined categories. 
Automizing this task can support designers in optimizing the arrangement of interactive elements, ensuring the creation of interfaces that are not only visually appealing but also functional and user-centric.
Classifiers can, e.g., assign semantic captions to elements, enable smart tutorials, or be the foundation for advanced search engines.
For this paper, we considered the ENRICO~\cite{leiva2020enrico} dataset as an example for layout design categorization.
ENRICO comprises 20 design categories and 1460 user interfaces with five modalities, including screenshots, wireframe images, semantic annotations, DOM-like tree structures, and app metadata.

\paragraph{Digit Recognition.}
We also include results for Audiovision-MNIST (AV-MNIST)~\cite{vielzeuf2018centralnet}, a multimodal dataset comprising images of handwritten and recordings of spoken digits.
Despite its apparent lack of immediate real-world application, the dataset's significance lies in its establishment as a standard multimodal benchmark. 
It allows us to situate our research within the broader context of previous research~\cite{perez2019mfas,lecun1998gradient}.

\begin{table}
\centering
    \begin{tabular}{lcccl}\toprule
        &\multicolumn{2}{c}{Baselines} & \multicolumn{2}{c}{Ours} \\
        \cmidrule(lr){2-3}\cmidrule(lr){4-5}
         {} & MultiRoute & MRO & MNL & InTense \\ 
        \midrule
        MUStARD & $65.9$ & $66.5$ & $67.4$ & $\mathbf{69.6}$ \\
        CMU-MOSI & $76.8$ & $75.8$ & $\mathbf{80.8}$ & $79.7$ \\
        UR-FUNNY & $63.6$ & $63.4$ & $63.4$& $\mathbf{65.1}$\\
        CMU-MOSEI & $80.2$ & $79.7$ & $80.5$ & $\mathbf{81.5}$\\
        AV-MNIST &71.8&72.0&72.4&\textbf{72.8}\\
        ENRICO &46.7&49.2&47.1&\textbf{50.8}\\
        \bottomrule
    \end{tabular}
     \caption{Accuracies for different baselines on the test fold. Each experiment is carried out ten times to compute the statistics.}
    \label{tab:performance}
\end{table}

\begin{figure}[t] 
	\centering
        \includegraphics[width=\columnwidth]{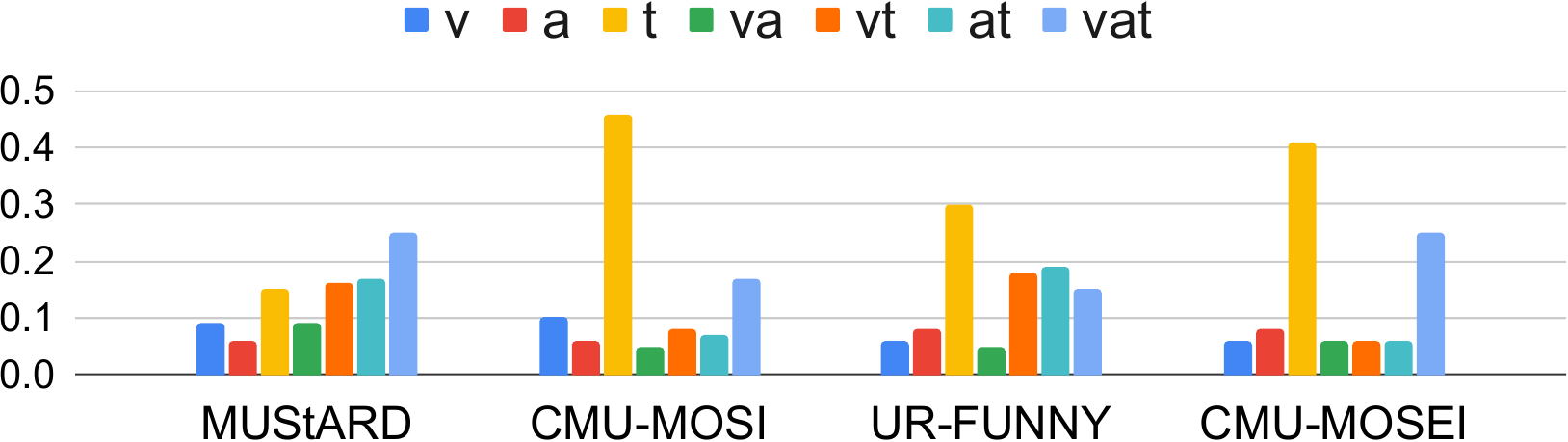}
     \caption{Relevance scores from InTense for audio (a), vision (v), text (t), and all their possible interactions.
     } 
     \label{fig:relevance}
\end{figure}

\paragraph{Baselines.} 
We compare the classification performance of \textit{InTense} and \acrshort{mnl} to the following state-of-the-art interpretable multimodal learning baselines: 1) Multimodal Residual Optimization (MRO)~\cite{wortwein2022beyond} and 2) Multimodal Routing (MultiRoute) ~\cite{tsai2020multimodal}. Additionally, we consider three non-interpretable baselines: 1) \textit{LF-Concat}, 2) \textit{TF Network}~\cite{zadeh2017tensor}, and 3) multimodal transformer (\textit{MulT})~\cite{tsai2019multimodal}. These baselines have been identified as leading in the independent comparison conducted by~\citeauthor{liang2021multibench}~\shortcite{liang2021multibench}. The multimodal transformer was particularly highlighted for consistently reaching some of the highest accuracy levels. 


\paragraph{Results.} 
The results are shown in Table~\ref{tab:performance}. We observe that our proposed models, MNL and InTense, surpass all interpretable baselines in terms of classification accuracy. InTense achieves the highest performance across all datasets except for CMU-MOSI, where MNL excels. CMU-MOSI is the smallest dataset in our analysis, a factor that may contribute positively to MNL's performance. Compared with non-interpretable multimodal learning methods (see Table 1 in Appendix E), MNL and especially InTense demonstrate impressive performance, almost meeting the accuracy of the non-interpretable Multimodal Transformer (MulT). The performance of our proposed models is within a narrow 2\% error margin (and frequently much lower) compared to the MulT baseline.


Figure~\ref{fig:relevance} shows the interpretable relevance scores that InTense assigns to the various modalities. Notably, in three of the four datasets analyzed, text emerges as the most significant modality. We identify two plausible explanations for this phenomenon. First, several studies have reported a strong correlation of text with sentiment~\cite{gat2021perceptual}. Second, the predominance of the text modality may be attributed to the availability of sophisticated word embeddings obtained from large pre-trained foundation models. However, we find an exception in the interpretability scores for the sarcasm detection dataset (MUStARD). Sarcasm detection requires information from multiple modalities, making sole reliance on one, especially text, insufficient for accuracy.

\section{Conclusion} 
We introduced InTense, a novel interpretable approach to multimodal learning that offers reliable relevance scores for modalities and their interactions. InTense achieves state-of-the-art performance in several challenging applications, from sentiment analysis and humor detection to layout design categorization and multimedia. We proved theoretically and validated empirically that InTense correctly disentangles higher-order interactions from the individual effects of each modality. The full transparency of InTense makes it suitable for future application in safety-critical domains.

\section{Broader Impact}
As an interpretable approach, the proposed methodology naturally aids in making multimodal learning more transparent. 
By attributing importance scores to different modalities and their interactions, InTense may reveal biases in decision-making and improve trustworthiness.
For instance, consider a system tasked with classifying loan suitability.
Our approach may expose social biases when relevance scores for certain modalities, such as gender extracted from vision, are disproportionately high. Moreover, unlike existing approaches, InTense has no higher-order interaction bias (see Section \ref{sec:intense}). That is, it does not incorrectly assign large relevance scores to higher-order interactions, which can create the false impression of a social bias. The full transparency of InTense prevents the deployment of a harmful classification model, contributing to the ethical use of AI in sensitive domains.

\section{Acknowledgement}
SV, PL, WM, and MK acknowledge support by the Carl-Zeiss Foundation, the DFG
awards KL 2698/2-1, KL 2698/5-1, KL 2698/6-1 and KL 2698/7-1, and the BMBF awards 03\textbar B0770E, and 01\textbar S21010C.

\bibliographystyle{named}
\bibliography{ijcai24}
\clearpage
\appendix
\title{Appendix}
\maketitle
\section{Proof of Theorem \ref{DasTheorem}}
We begin with a few simple optimization lemmas. 

\begin{lemma}
	\label{lemmaone}
	Let $u>0$, we have 
	\begin{align}
	\argmin_{x>0}\frac{u}{x}+\frac{\lambda}{p}x^p&=(\frac{u}{\lambda})^\frac{1}{p+1}\\
		\min_{x>0}\frac{u}{x}+\frac{\lambda}{p}x^p&=\frac{p+1}{p}u^{\frac{p}{p+1}}\lambda^{\frac{1}{p+1}}.
	\end{align}
	\end{lemma}

\begin{proof}
	Let 
$$f(u,x)=\frac{u}{x}+\frac{\lambda}{p}x^p, \text{we have}$$ 
$$\frac{\partial f}{\partial x}=-\frac{u}{x^2}+\lambda x^{p-1},$$ 
which cancels at $x^*=(\frac{u}{\lambda})^{\frac{1}{p+1}}$. Note also that $\frac{\partial f}{\partial x}\leq 0$ for $x\leq x^*$ and $\frac{\partial f}{\partial x}\geq 0$ for $x\geq x^*$, guaranteeing a minimum. 
		Substituting back into the definition of $f$, we obtain
  $$f(x,u^*)=u^{\frac{p}{p+1}}\lambda^{\frac{1}{p+1}}+\frac{\lambda}{p}\frac{u^{\frac{p}{p+1}}}{\lambda^{\frac{p}{p+1}}}=\frac{p+1}{p}u^{\frac{p}{p+1}}\lambda^{\frac{1}{p+1}}$$
\end{proof}

\begin{lemma}
	\label{lemmatwo}
	Let $A,u,B>0$ with $u<1$, we have 
	\begin{align}
	\argmax_{x>0}Ax^u-Bx&=(\frac{Au}{B})^{\frac{1}{1-u}}\\
	\max_{x>0}Ax^u-Bx&=A(\frac{Au}{B})^{\frac{u}{1-u}}-B(\frac{Au}{B})^{\frac{1}{1-u}}.
	\end{align}
\end{lemma}
\begin{proof}
	Let $f(A,B,u,x)=Ax^u-Bx$. We have $$\frac{\partial f}{\partial x}=Aux^{u-1}-B.$$ This derivative cancels at $x=x^*:=(\frac{Au}{B})^{\frac{1}{1-u}}$ as expected. Substituting back into the definition of $f$, we obtain $$f(A,B,u,x^*)=A(\frac{Au}{B})^{\frac{u}{1-u}}-B(\frac{Au}{B})^{\frac{1}{1-u}},$$ as expected.
\end{proof}

We can now write the main step in the proof of Theorem~\ref{DasTheorem}:
\begin{prop}
	\label{Theprop}
	Let $A_1,A_2,\ldots,A_m>0$, we have 
	\begin{align}
	\min_{\beta>0,\|\beta\|_p^p\leq 1}\sum_{m=1}^M \frac{A_m}{\beta_m}=(\sum_{m=1}^M A_m^{\frac{p}{p+1}})^{\frac{p+1}{p}},
	\end{align}
	and the $\beta_m$'s which satisfy the above minimization problem are given by 
		\begin{align}
\beta_m= \frac{         A_m^{\frac{1}{p+1}}       }{    (\sum_{m=1}^M A_m^{\frac{p}{p+1}})^{\frac{1}{p}}         }.
		\end{align}
	\end{prop}
\begin{proof}
    Let us write $f(A,B)=\sum_{m=1}^M \frac{A_m}{\beta_m}$. The Lagrangian is $\mathcal{L}(\lambda,f,\beta)=\sum_{m=1}^M  \frac{A_m}{\beta_m} +\frac{\lambda}{p}(\|\beta\|_p^p-1)$.
    By Lemma~\ref{lemmaone}, optimizing over each component of $\beta$ independently, we have 
    \begin{align}
    \min_\beta \mathcal{L}(\lambda,f,\beta)=\sum_{m=1}^M \frac{p+1}{p}A_m^{\frac{p}{p+1}}\lambda^{\frac{1}{p+1}}-\frac{\lambda}{p},
    \end{align}
    realized for $\beta_m=(\frac{A_m}{\lambda})^{\frac{1}{p+1}}$.
    Hence, the dual problem is 
    \begin{align}
    \max_\lambda \sum_{m=1}^M \frac{p+1}{p}A_m^{\frac{p}{p+1}}\lambda^{\frac{1}{p+1}}-\frac{\lambda}{p}.
    \end{align}
    By Lemma~\ref{lemmatwo}, we obtain
            \begin{align}
    &\max_\lambda \sum_{m=1}^M \frac{p+1}{p}A_m^{\frac{p}{p+1}}\lambda^{\frac{1}{p+1}}-\frac{\lambda}{p}\\
    &=\sum_{m=1}^M A_m^{\frac{p}{p+1}}\frac{p+1}{p}(\sum_{m=1}^M A_m^{\frac{p}{p+1}})^{\frac{1}{p}}-\frac{(\sum_{m=1}^M A_m^{\frac{p}{p+1}})^{\frac{p+1}{p}}}{p}\\
    &=(\sum_{m=1}^M A_m^{\frac{p}{p+1}})^{\frac{p+1}{p}},
    \end{align}
    realized for $\lambda=(\sum_{m=1}^M A_m^{\frac{p}{p+1}})^{\frac{p+1}{p}}$.
    
    Note that for $m=1,2,\ldots,M$, 
    \begin{align}
    \label{TheBeta}
    \beta_m=(\frac{A_m}{\lambda})^{\frac{1}{p+1}}=\frac{A_m^{\frac{1}{p+1}}}{(\sum_{m=1}^M A_m^{\frac{p}{p+1}})^{\frac{1}{p}}}.
    \end{align}
    
    Finally, note that 
    \begin{align*}
    \sum_{m=1}^M \beta_m^p&=\frac{\sum_{m=1}^M A_m^{\frac{p}{p+1}}}{(\sum_{m=1}^M A_m^{\frac{p}{p+1}})^{\frac{p}{p}}}=1.
    \end{align*}
    Hence, complementary slackness is satisfied, and we have strong duality.
    \end{proof}

We can now proceed with the 
\begin{proof}[Proof of theorem~\ref{DasTheorem}]
    
First of all, similarly to~\cite{kloft2011lp} and Lemma 26 in~\cite{micchelli2005learning} and  we apply the following substitution: $w^m_L \leftarrow \sqrt{\beta_m} w^m_L$, which turns the problem into the following: 
	

\begin{multline}
\label{Thiss}
\minimize_{\substack{w_L^1, \cdots ,w_L^M, \beta, \\ W^1,W^2, \cdots, W^M,\\ \beta\in \sR^M,\beta\geq 0, \|\beta\|_p\leq 1}}\Bigg(\sum_{i=1}^n \ell\Big(\sum_{m=1}^M \langle w_L^m,f^{m}(x_i^m)\rangle
+b,y_i\Big)\\
+\Lambda \sum_{l=1}^{L-1}\sum_{m=1}^M \|w^m_l\|_2^2
+\Lambda\sum_{m=1}^M \frac{\|w^m_L\|_2^2}{\beta_m}\Bigg).
\end{multline}

The minimization problem from equation~\ref{Thiss} can be then solved for $\beta$ independently, removing the dependence on $\beta$: simply apply Proposition~\ref{Theprop} for $A_m=\|\omega_L^m\|_2^2\Lambda$ to yield the equivalence to 

\begin{align}
\label{eq:deepmkl}
\minimize_{\substack{w_L^1,w_L^2,\ldots,w_L^M, \\ W^1,W^2, \cdots, W^M}}\sum_{i=1}^n \ell\left(\sum_{m=1}^M  \langle w_L^m,f^{m}(x^m_i)\rangle +b,y_i\right) \nonumber \\
+ \Lambda \sum_{l=1}^{L-1}\sum_{m=1}^M \|w^m_l\|_2^2+\Lambda\left(\sum_{m=1}^M \|w^m_L\|_2^q\right)^{\frac{2}{q}},
\end{align}
where $q=\frac{2p}{p+1}$, as expected.

\end{proof}

	\section{Proof of Theorem~2}
	  \begin{proof}

   The first statement is a straightforward consequence of equation \ref{eq:centering}, collecting all the terms involving the same products of components: indeed, we can just write 
   \begin{align}
    \mathcal{G}_J=\sum_{\ell=0}^{|I|}(-1)^{\ell}\hspace{-1em} \sum_{\emptyset \neq S^1,\ldots,S^{\ell}\subset I \atop \substack{ S^1,\ldots,S^{\ell} \text{disjoint} \\ \bigcup_{k\in \{ 1,2,\ldots, \ell\}} S^k =I\setminus J}} \hspace{-0.5em}\prod_{k\in \{ 1,2,\ldots, \ell\}}  \mathbb{E}\left(\prod_{m\in S^k}f^{m}\right).
   \end{align}
The proof of the Equation~\ref{eq:maintheorem} is by induction over the cardinality of $I$.

\paragraph{Initial case.}
If $|I|=1$, we have $\Psi(f)=f-\mathbb{E}(f)$. The only admissible $J$ is the empty set, and we indeed have $$\mathbb{E}\left(    \sum_{K: J\subset K \subset I} \mathcal{G}_{K}  \prod_{m\in K\setminus J}f^{m}             \right)=\mathbb{E}\left(f-\mathbb{E}(f)\right)=0.$$

\paragraph{Induction step.}
    
    We suppose that the result holds for all cardinalities strictly less than $|I|$. 
    Let us suppose first that $J\neq \emptyset$.
    Note that by the construction of $\Psi$, we have immediately 
    \begin{align*}
    \sum_{K: J\subset K \subset I} \mathcal{G}_{K}  \prod_{m\in K\setminus J}f^{m}           = \Psi_{I\setminus J}(f).
    \end{align*}
    Since $J\neq \emptyset$, we have $|I\setminus J|<|I|$, and by the induction hypothesis, $\mathbb{E} \left( \Psi_{I\setminus J}(f)\right) =0$. Hence 
       \begin{align*}
    \sum_{K: J\subset K \subset I} \mathcal{G}_{K}  \prod_{m\in K\setminus J}f^{m}           = \Psi_{I\setminus J}(Z)=0,
    \end{align*}
    which concludes the induction step in the case $J\neq \emptyset$.
    
    Finally, if $J=\emptyset$, we have 
    \begin{equation*}
    \mathbb{E}\left( \sum_{ K \subset I} \mathcal{G}_{K}  \prod_{m\in K}f^{m}              \right)=\mathbb{E}\left(\Psi_{I}(f)\right)
    \end{equation*}
    \begin{multline*}
    =\mathbb{E}\Biggl(\sum_{\ell=0}^{|I|}(-1)^{\ell} \sum_{\emptyset \neq S^1,S^2,\ldots,S^{\ell}\subset I \atop  S^1,S^2,\ldots,S^{\ell} \text{disjoint}} \prod_{m \in I \setminus (\cup S^{k})}f^{m} \prod_{k\in \{ 1,2,\ldots, \ell\}}\\
    \mathbb{E}\biggl(\prod_{m\in S^k}f^m\biggl)\Biggl)
    \end{multline*}
    \begin{align*}
    &=\sum_{\ell=0}^{|I|}(-1)^{\ell} \sum_{\emptyset \neq S^1,S^2,\ldots,S^{\ell}\subset I \atop  S^1,S^2,\ldots,S^{\ell} \text{disjoint}}\mathbb{E}\left(\prod_{m \in I \setminus (\cup S^{k})}f^{m}\right)\prod_{k\in \{ 1,2,\ldots, \ell\}} \\ &\qquad\qquad\qquad\qquad\qquad\qquad\qquad\qquad\mathbb{E}\left(\prod_{m\in S^k}f^m\right) \\
    &=\sum_{\ell=0}^{|I|-1}(-1)^{\ell} \sum_{\emptyset \neq S^1,S^2,\ldots,S^{\ell+1} \text{disjoint} \atop  \cup_{k=1}^{\ell+1}S^k=I }  \prod_{k\in \{ 1,2,\ldots, \ell+1\}} \mathbb{E}\left(\prod_{m\in S^k}f^m\right) \\
     &\qquad\;+\sum_{\ell=1}^{|I|}(-1)^{\ell} \sum_{\emptyset \neq S^1,S^2,\ldots,S^{\ell} \text{disjoint} \atop  \cup_{k=1}^{\ell}S^k=I }  \prod_{k\in \{ 1,2,\ldots, \ell\}} \mathbb{E}\left(\prod_{m\in S^k}f^m\right)\\
    &=\sum_{\ell=1}^{|I|}(-1)^{\ell-1} \sum_{\emptyset \neq S^1,S^2,\ldots,S^{\ell} \text{disjoint} \atop  \cup_{k=1}^{\ell+1}S^k=I }  \prod_{k\in \{ 1,2,\ldots, \ell\}} \mathbb{E}\left(\prod_{m\in S^k}f^m\right) \\
     &\qquad +\sum_{\ell=1}^{|I|}(-1)^{\ell} \sum_{\emptyset \neq S^1,S^2,\ldots,S^{\ell} \text{disjoint} \atop  \cup_{k=1}^{\ell}S^k=I }  \prod_{k\in \{ 1,2,\ldots, \ell\}} \mathbb{E}\left(\prod_{m\in S^k}f^m\right)\\
     &=0,
    \end{align*}
    where at the fourth and fifth lines, we have split the sum into two cases according to whether $\cup_{k=1}^{\ell}S^k=I$.
\end{proof}


\subsection{Centering step}
To make the exposition clearer, we provide as an example the case $I=\{1,2,3\}$ and the features $f^m$ are one dimensional. To do this, let us first observe that in equation~\ref{eq:centering}, several terms of the sum are identical to each other. Indeed, if we rename the sets $S^{1},\ldots,S^{\ell}$ by $S^{\sigma(1)},\ldots,S^{\sigma(\ell)}$ where $\sigma$ is any permutation of $\{1,2,\ldots,\ell\}$, then we still obtain a valid term. In particular, collecting all of the terms which are equal to each other, we can rewrite the formula equivalently as follows: 
\begin{align*}
	\Psi_{I}(f)&=\sum_{\ell=0}^{|I| }(-1)^{\ell} \sum_{\emptyset \neq S^1,\ldots,S^{\ell}\subset I \atop  S^1,\ldots,S^{\ell} \text{disjoint}}  \mathbb{E}\left(\prod_{m\in S^k}f^{m}\right)\\
 &\qquad \qquad \qquad \qquad \qquad \quad \quad \quad \prod_{m \in I \setminus (\cup S^{k})}f^m\\
    &=\sum_{\ell=0}^{|I|}(-1)^{\ell} \ell! \sum_{\emptyset \neq S^1<\ldots<S^{\ell}\subset I \atop S^1,\ldots,S^l \text{disjoint} }\mathbb{E}\left(\prod_{m\in S^k}f^{m}\right)\\
    &\qquad \qquad \qquad \qquad \qquad \quad \quad \quad\prod_{m \in I \setminus (\cup S^{k})}f^m,
	\end{align*}
where `$<$' denotes any strict order relation between subsets fixed in advance. 

In this case, we have, using the notation $\overline{f^1f^2}=\mathbb{E}(f^1f^2)$ etc.: 
    \begin{align*}
    \hat{f}^{1,2,3}&=f^1\times f^2\times f^3\;-\\
    &\quad\overline{f^1}\times f^2\times f^3-f^1\times \overline{f^2}\times f^3-f^1\times f^2\times \overline{f^3}\;-\\
    &\quad\overline{f^1\times f^2}\times f^3-\overline{f^1\times f^3}\times f^2-\overline{f^2\times f^3}\times f^1\;-\\
    &\quad\overline{f^1\times f^2\times f^3}\;+\\
    &\quad2\times \overline{f^1} \times f^3\times  \overline{f^2}+2\times \overline{f^2}\times f^1\times  \overline{f^3}\;+\\
    &\quad2\times \overline{f^1}\times f^2\times \overline{f^3}+2 \times \overline{f^1\times f^2} \times \overline{f^3}\;+\\
    &\quad2\times \overline{f^1\times f^3}\times  \overline{f^2}+2\times \overline{f^2\times f^3} \times \overline{f^1}\;-\\
    &\quad6\times \overline{f^1}\times \overline{f^2}\times \overline{f^3}.
    \end{align*}

\section{Synthetic-Datasets}
We describe in detail the synthetic datasets we used in our main paper. 
We created two synthetic datasets with a common multimodal structure where each modality of a datapoint is a sequence of letters chosen randomly from a predefined set $\{A, C, G, T\}$. 
For modality $m$, each datapoint is a sequence of length $100$.
We created two specific information carrying sub-sequences of length 7---\textsc{ACGTAGC} and \textsc{GATGTAC}. 
For a mulimodal datapoint, an information subsequence is inserted at a random position with a probability $p_m$.
We created $10k$ multimodal samples and divided them into two equal halves.
The first half contains the first sub-sequence (\textsc{ACGTAGC}) inserted in each datapoint's modality with probability $p_m$, while the second half contains the second sub-sequence (\textsc{ACGTAGC}) inserted in each datapoint's modality with the same probability. 
Thus, the number of samples where the subsequence is inserted is approximately equal in the two parts across all the modalities.
Same can be said for the number of datapoints in the two halves where no information is inserted.
A high value of $p_m$ indicates that a vast majority of datapoints in modality $m$ contain information specific subsequences.

\subsection{Dataset with Independent Modalities}
We created the \textsc{SynthGene} dataset where the datapoints in the first half mentioned above are labeled as positive samples, and all the datapoints in the second half are labeled as negative. Thus, the dataset contains equal number of positive and negative samples.
\textsc{SynthGene} contains 10 modalities $M_{1:10}$, where modalities $M_2$ is created with $p_m=0.2$, $M_4$ with $p_m=0.3$, and  $M_7$ with $p_m=0.5$. The other modalities are created with $p_m=0$ to make them irrelevant to the classification problem. However, there could be some informative sequences in these modalities due to randomness. A high value of $p_m$ here indicates that more class specific sequences are present in that modality. Thus, the higher the value of $p_m$, the more relevant modality $m$ becomes.

\subsection{Dataset with Modality Interactions}
We also create the \textsc{SynthGene-tri} dataset, a trimodal version of \textsc{SynthGene}. However, this time, the informative subsequence is not class-specific. Instead, the label is defined by an exclusive-or (XOR) relationship between the first two modalities ($M_1$ and $M_2$). The label is $0$ if both modalities contain the subsequence or none of them does, and the label is $1$ otherwise (i.e., when one of the modalities contains the subsequence).
Note that modality $M_3$ does not contain any informative subsequence; thus, it is irrelevant.
As before, we generate a balanced dataset with 50\% of the samples being positive and 50\% negative.

\section{Real-world Datasets}

\paragraph{Motivation.}
We have chosen four datasets to study how humans perceive emotions from different data sources like sound, images, and text. These datasets are crucial for making computers understand emotions, analyzing human behavior, and improving AI-driven education. These datasets often provide complementary information that becomes clear when considering different modalities. The modalities of the datasets are intuitive for humans because humans well perceive the modalities in their daily experiences. Thus, these datasets are a great fit for our approach to studying and interpreting how different modalities interact to predict human behavior.

\paragraph{Datasets Preprocessing.}
To ensure reproducibility and allow performance comparison, we follow the data preprocessing used in the MultiBench \cite{liang2021multibench} benchmark for all used datasets.

\subsection{MOSI}
CMU Multimodal Opinion-level Sentiment Intensity (CMU-MOSI) dataset~\cite{zadeh2016mosi} comprises multimodal observations obtained from rigorously transcribed opinion video content from the YouTube platform, focusing on video blogs (vlogs). These videos contain individuals expressing opinions on various subjects while primarily maintaining direct eye contact with the camera. The transcriptions provide detailed and highly relevant sentiment-carrying information, such as pause fillers, stresses, and speech pauses.

\paragraph{Preprocessing Details.} We adopt the preprocessing steps detailed by \cite{liang2018computational,bager-zadeh-etal-2018-multimodal}:

\begin{enumerate}
    \item \textbf{Text (t)}: The spoken text is represented by a sequence of 300-dimensional word embeddings obtained by Global Vectors for Word Representation (GloVe)~\cite{pennington2014glove} trained on common crawl dataset. Timing of word utterances is obtained with the Penn Phonetics Lab Forced Aligner (P2FA).
    \item \textbf{Vision (v)}: We use visual features extracted by the Facet library ~\cite{de2011facial} and OpenFace facial behavior analysis toolkit ~\cite{baltruvsaitis2016openface}. Applying these models results in the following information:  facial action units, facial landmarks, head pose, gaze tracking, HOG features, and facial expression features.
    \item \textbf{Audio (a)}: As audio input to our model, we use a sequence of acoustic features extracted by COVAREP~\cite{degottex2014covarep}. The features include 12 Mel-frequency cepstral coefficients, pitch tracking and voiced/unvoiced segment features, glottal source parameters, peak slope parameters, and maxima dispersion quotients.
\end{enumerate}

\subsection{MOSEI}
CMU Multimodal Opinion Sentiment and Emotion Intensity (CMU-MOSEI) ~\cite{bager-zadeh-etal-2018-multimodal} is the largest dataset for sentiment analysis and emotion recognition. The dataset consists of annotated videos collected from Youtube of individuals discussing a topic of interest in front of the camera. Each video is annotated for sentiment and the presence of 9 discrete emotions (angry, excited, fearful, sad, surprised, frustrated, happy, disappointed, and neutral) and continuous emotions (valence, arousal, and dominance).

\paragraph{Preprocessing Details.} We use a preprocessing approach similar to that employed for CMU-MOSI and UR-FUNNY datasets.

\subsection{MUStARD} 
MUStARD is a multimodal dataset that has been developed for the purpose of advancing automated sarcasm detection research~\cite{castro2019multimodal}. This dataset is curated from renowned television shows such as Friends, The Golden Girls, The Big Bang Theory, and Sarcasmaholics Anonymous. MUStARD contains audiovisual statements that have been labeled to denote instances of sarcasm. Each statement is accompanied by its contextual background, contributing supplementary insights into the circumstances surrounding the statement's occurrence. This, in turn, poses an additional challenge in effectively modeling multimodal information over extended periods. The deliberate selection of sarcasm as the annotation task is motivated by the need for intricate modeling of complementary details, especially in cases where semantic cues from different modalities may not align.

\paragraph{Dataset Preprocessing.} We use the preprocessing steps suggested in the original paper~\cite{castro2019multimodal}, which goes as follows for each modality:

\begin{enumerate}
    \item \textbf{Text (t)}: The embeddings for the text modality are obtained from pretrained BERT~\cite{devlin2018bert} representations and from Common Crawl pretrained 300-dimensional GloVe word vectors for each token.
    \item \textbf{Vision (v)}: Visual features are extracted for each frame using the pool 5 layers of ResNet-152 that is pre-trained on the Imagenet dataset. Every frame is preprocessed by resizing, center cropping, and normalizing. The OpenFace analysis tool has also been used to extract facial expression features.

    \item \textbf{Audio (a)}: Low-level features from the audio data stream are extracted using the speech processing library Librosa. Additionally, COVAREP is used to extract acoustic features, including 12 Melfrequency cepstral coefficients, pitch tracking and voiced/unvoiced segment features, glottal source parameters, peak slope parameters, and maxima dispersion quotients.
\end{enumerate}

\subsection{UR-FUNNY}
UR-FUNNY~\cite{hasan2019ur} is the first large-scale multimodal dataset of annotated video excerpts extracted from TED talks for humor detection in human speech. The essence of humor, as an intricate communicative tool, is inherently multimodal, interweaving articulate language (text), complementary gestures (visual), and expressive prosodic cues (acoustic).

\paragraph{Preprocessing Details.} We use a preprocessing approach similar to that employed for CMU-MOSI and UR-FUNNY datasets.

\section{Experimental Setup}
In our experimental Setup, we compare the baselines methods against our two methods: 1) \acrfull{mnl} that captures the linear combinations of the modality representations in an interpretable manner, 2) \acrlong{itf} that can capture modality combinations and their interactions in an interpretable manner. We provide an implementation of these methods and our appendix, which we plan to release to the public community.

\subsection{Baselines}

\begin{table*}
\centering
    \begin{tabular}{lccccccl}\toprule
        & \multicolumn{3}{c}{Non-Interpretable} & \multicolumn{4}{c}{Interpretable} \\
        \cmidrule(lr){2-4}\cmidrule(lr){5-8}
         {} & LF-Concat & TF Network & MulT & MultiRoute & MRO & MNL & InTense \\ 
        \midrule
        MUStARD & $65.4\pm1.0$ & $62.1\pm0.8$& $71.2\pm0.4$ & $65.9\pm0.3$ & $66.5\pm0.6$ & $67.4\pm0.5$ & $\mathbf{69.6\pm0.4}$ \\
        CMU-MOSI & $75.2\pm0.8$ & $74.4\pm0.2$ & $82.7\pm0.3$& $76.8\pm0.8$ & $75.8\pm0.6$ & $\mathbf{80.8\pm0.8}$ & $79.7\pm0.6$ \\
        UR-FUNNY & $62.2\pm0.6$ & $61.2\pm0.4$ & $66.2\pm0.1$& $63.6\pm0.9$ & $63.4\pm0.4$ & $63.4\pm0.3$& $\mathbf{65.1\pm0.6}$\\
        CMU-MOSEI & $79.2\pm0.4$ & $79.4\pm0.5$ & $82.1\pm0.7$ & $80.2\pm0.8$ & $79.7\pm0.6$ & $80.5\pm0.4$ & $\mathbf{81.5\pm0.3}$\\
        
        \bottomrule
    \end{tabular}
     \caption{Accuracies for different baselines on the test fold. Each experiment is carried out ten times to compute the statistics.}
    \label{tab:performance_appendix}
\end{table*}
We compare our method against three non-interpretable baselines and two interpretable baselines, which we describe below with key points over their implementation.

\paragraph{Non-interpretable Baselines.} We consider the following three  baselines:

\textbf{LF-Concat} is a non-interpretable baseline, where the output of the encoder modules is concatenated, followed by a fully-connected layer for prediction. We chose this baseline as it is very near to our own implementations. We provide our own implementation of this method.

\textbf{TF-Network.} from \cite{zadeh2017tensor} is non-interpretable method used as a baseline. We chose this baseline as InTense uses a tensor fusion module to capture modality interactions. We provide an implementation of the TF Network in our code adopted from the original implementation of the paper.

\textbf{MulT.} Multimodal transformers from \cite{tsai2019multimodal} is a non-interpretable method used as the third baseline. We chose this as a  baseline as it is the best-performing model on the chosen four datasets according to \cite{liang2021multibench}. We provide an implementation of this method in our code adopted from the original implementation of the paper.

\paragraph{Interpretable Baselines.} We consider the following two baselines for a comparison:

\textbf{MRO.} Multimodal Residual Optimization (MRO) from \cite{wortwein2022beyond} is considered as a baseline as it aims at disentangling the uni-, bi-, and tri-modal interactions. We provide an implementation of this method in our code repository.

\textbf{Multiroute.} Multimodal Routing from \cite{sabour2017dynamic} is considered a baseline as it aims to interprect the modalities and their interactions. We use the implementation provided from the authors to evaluate the results on the four datasets.

\subsection{Encoder Modules} \label{sec:encoder}
We first apply a standard batch normalization over the input features of each modality, followed by a GRU (Gated Recurrent Unit) module for each modality. We use the same encoder for all experiments to test our baselines except MulT (Multimodal Transformers), as it relies on the multimodal transformer architecture. We report MultiRoute as well using its original implementation.

\subsection{Network Architecture}
Figure \ref{fig:concept_fig} summarizes the network architecture for InTense, with the tensor fusion modules to capture modality interactions and the Iterative Batch Normalization module to untangle modality interactions from individual modality effects.

\begin{figure}[h]
    \centering
    \includegraphics[width=\columnwidth]{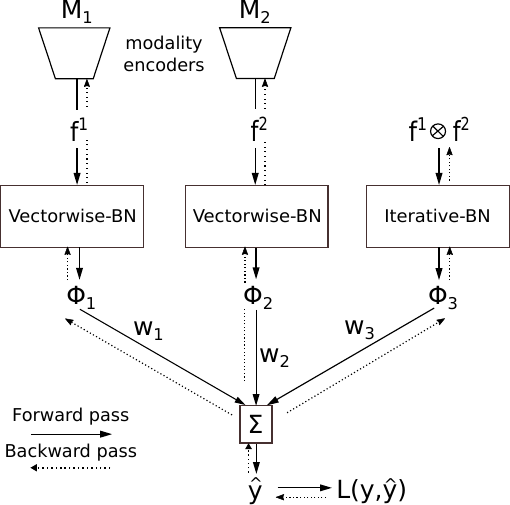}
    \caption{The figure shows the framework of InTense for two modalities. The modalities and their fusion are jointly trained to obtain the representations ($f^1$, $f^2$), which then undergo through our TensorFusion module to capture the multiplicative interactions. We apply IterBN to disentangle the multiplicative interactions from individual modality effects. All the modality representations and their multiplicative interactions are fused together in the end to obtain the final prediction.  }
    \label{fig:concept_fig}
\end{figure}

\subsection{Performance Evaluation}
We use four multimodal datasets that are evaluated in two different ways. These are explained below.

\paragraph{1. MOSI, 2. MOSEI.} These two datasets are trained using the MAE (Mean Absolute Error) objective on sentiment analysis, where the labels are integers between $-3$ and $3$. The model with the least MAE on the validation set is chosen during training. During testing, the accuracy is calculated with the positive/negative class setting, where a prediction less than $0$ is considered a negative class, and a prediction greater than $0$ is considered a positive class. The data points with label $0$  are considered as neutral are not considered while evaluating accuracy on the test set.

\paragraph{3. SARCASM, 4. UR-FUNNY.} These are the binary-classification datasets with labels $0$ and $1$. The training objective is Cross-Entropy Loss. The model with the least loss on the validation set is chosen during training. The accuracy is computed on the test fold using the best model obtained on the validation set.

\paragraph{5. AV-MNIST, 6. ENRICO.} These are the multi-class classification datasets containing  $10$ classes for AV-MNIST and $23$ classes for ENRICO daaset
The training objective is Cross-Entropy Loss.
The model with the least loss on the validation set is chosen during training.
The accuracy is computed on the test fold using the best model obtained on the validation set.

For all the experiments, we report the mean accuracy and its standard deviation on ten random seeds.
\begin{table*}
    \centering
    \begin{tabular}{lcccccl}\toprule
    {} & {} & lr & batch\_size & hidden\_dims & tf\_indices & tf\_latent\_dim\\ \midrule
    \multirow{2}{*}{MOSEI}&MNL&$1e-3$&$32$&$[64,64,768]$&None&None \\ 
    &InTense&$1e-3$&$32$&$[64,64,1024]$&['13']&8 \\ \midrule
    \multirow{2}{*}{MOSI}&MNL&$5e-3$&$32$&$[64,64,768]$&None&None \\ 
    &InTense&$1e-3$&$32$&$[64,64,768]$&['13']&8 \\ \midrule
    \multirow{2}{*}{UR-FUNNY}&MNL&$5e-3$&$32$&$[64,64,768]$&None&None \\ 
    &InTense&$1e-3$&$32$&$[64,64,768]$&['13']&8 \\ \midrule
    \multirow{2}{*}{MUStARD}&MNL&$5e-3$&$32$&$[64,64,768]$&None&None \\ 
    &InTense&$1e-3$&$32$&$[64,64,768]$&['13']&8 \\ \midrule
    \end{tabular}
    \caption{The table shows the best performing hyperparamaters for MNL and InTense across all four real-world dataset}
    \label{tab:hyper}
\end{table*}
\subsection{Hyperparameters} \label{sec:hparams}

Below, we will delve into the key hyperparameters used to obtain the reported performance in InTense. In our implementation, we have used the same nomenclature for these hyperparameters to make our code easier to follow. Table \ref{tab:hyper} shows a list of the hyperparameter settings we used to achieve the best results across all baseline models.

\begin{enumerate}
\item \textbf{learninng rate (lr)}: For all the experiments, we chose the initial learning rate between $0.0005$ to $0.005$. We used an exponential learning rate decay that decays the learning rate of each parameter group by gamma every epoch. The value of gamma is set between $0.9$ and $1$.

\item \textbf{epochs}: We ran all our experiments for $100$ epochs and chose the model with the best validation metrics. The validation metrics for the datasets are shown in Table \ref{tab:hyper}

\item \textbf{batch\_size}: Batch size denotes the number of datapoints in a mini-batch. We chose a batch size of either $16$ or $32$ for all the experiments.

\item \textbf{weight\_decay}: As per our optimization framework defined in Equation \ref{eq:deepmkl} of the main paper, a weight decay is applied to every modality encoder. The value of weight decay was set to $0.01$ for all the experiments.

\item \textbf{reg\_param}: It is the hyperparameter $\Lambda$ in Equation \ref{eq:deepmkl} of our framework that tunes the block-norm regularization for the fusion layer. We set its value to $0.05$ for all the experiments.

\item \textbf{p}: It defines the value of the $q-$norm over the weights of the fusion layer (described in the Equation \ref{eq:deepmkl} of the main paper). The value can be used to tune the sparsity of relevance scores with $1 \leq q \leq 2$. We set the value of this hyperparameter to be $1$ for all our experiments.

\item \textbf{hidden\_dims}: This hyperparameter is an array with dimensions matching the count of distinct modalities in the dataset. Each individual element within the array determines the resultant dimensionality of a corresponding modality encoder module. Table \ref{tab:hyper} shows a list of the hyperparameter settings.

\item \textbf{tf\_indices}: This hyperparameter is an array of modality indices that defines the interactions captured by InTense. E.g. a value $['12','23']$ would indicate an interaction between modalities at index $1$,$2$ and  an interaction between modalities at index $2$,$3$ is learned. 

\item \textbf{tf\_latent\_dim}: This hyperparameter is an integer value that defines the dimension of the modality representations undergoing the tensor product. To alleviate computation overhead, the dimension of the output from the modality encoders is mapped to a size of \textit{tf\_latent\_dim}.

\item \textbf{scheduler\_gamma}: Decays the learning rate of each parameter group by gamma every epoch. The default value for all the experiments is set to $0.9$.

\item \textbf{optimizer}: Set the optimizer for learning. By default, we used PyTorch implementation of the AdamW algorithm.

\item \textbf{task}: Defines the task for which the model is trained. The task is set regression for the dataset ``MOSI" and ``MOSEI" where the labels are regression labels. For other datasets, the default value of \textit{task} is set to ``classification."

\item \textbf{act\_before\_vbn}: In our VBN normalization scheme, the output from the modality encoders are passed through ReLU activation when this parameter is set to True. Otherwise, the VBN normalization is directly applied to the outputs of the modality encoders. We found through experiments that setting this parameter doesn't improve the performace, so the value is set to False for all the experiments.
\end{enumerate}


\section{Ablation Studies}
\begin{table*}
    \centering
    \begin{tabular}{l|cc|cc|cl}\toprule
    {} & \multicolumn{2}{c}{GRUwithBN} & \multicolumn{2}{c}{Transformers}  & \multicolumn{2}{c}{MLP} \\ \midrule
    {} & MNL & InTense & MNL & InTense & MNL & InTense   \\ \midrule
    MOSI &\bolden{$80.8\pm0.8$}&$79.7\pm0.6$&$78.2\pm0.2$&$79.5\pm0.4$&$77.4\pm0.6$&$78.2\pm0.4$\\
    MOSEI &$80.5\pm0.4$&\bolden{$81.5\pm0.3$}&$80.2\pm0.5$&$81.2\pm0.2$&$78.8\pm0.7$&$79.7\pm0.2$\\
    UR-FUNNY &$64.4\pm0.3$&$65.1\pm0.6$&$65.6\pm0.3$&\bolden{$65.8\pm0.3$}&$62.6\pm0.3$&$63.9\pm0.5$\\
    MUStARD &$67.4\pm0.5$&$69.6\pm0.4$&$67.8\pm0.5$&\bolden{$69.8\pm0.2$}&$65.8\pm0.4$&$66.4\pm0.3$\\
    \bottomrule
    \end{tabular}
    \caption{The table shows the Accuracy of our methods MNL and InTense on the test set. Here we use three different encoders to perform an evaluation of our method. We report the mean accuracy and its standard deviation over $10$ experiments performed on random seeds.}
    \label{tab:ablation}
\end{table*}

\begin{table*}
    \centering
    \begin{tabular}{l|cc|cc|cc|cl}\toprule
    {} & \multicolumn{2}{c}{MOSI} & \multicolumn{2}{c}{MOSEI}  & \multicolumn{2}{c}{MUStARD} & \multicolumn{2}{c}{UR-FUNNY}\\ \midrule
    {} & Acc. & Rel. & Acc. & Rel. & Acc. & Rel. & Acc. & Rel.  \\ \midrule
    text(t) &\textbf{74.2}&{\textbf{0.75}}&\textbf{78.8}&{\textbf{0.85}}&\textbf{68.6}&\textbf{0.78}&\textbf{58.3}&{\textbf{0.77}}\\
    audio(a) &65.5&{0.11}&66.4&{0.05}&64.9&{0.07}&57.2&{0.15}\\
    video(v) &66.3&{0.14}&67.2&{0.10}&65.7&{0.16}&57.3&{0.07}\\
    \bottomrule
    \end{tabular}
    \caption{The table shows that the modality with highest accuracy individually gets a high relevance score when trained through InTense. \textit{Acc.} stands for the Accuracy of the model on the test set, and \textit{Rel.} stands for the Relevance scores obtained through InTense}
    \label{tab:correlation}
\end{table*}
\subsection{What can go wrong}
We conducted experiments where we found that without proper normalization, the interpretability scores are biased toward higher-order interactions, even when those interactions do not add any useful information as compared to the lower-order modality interactions among fewer modalities.

To conduct our experiments, we considered the \textsc{SynthGene-tri} dataset. Within this dataset, we have established that while labels are influenced solely through modality interaction between modalities $1$ and $2$, no valuable information for label determination is present in modality $3$. Therefore, an interaction involving fewer modalities ($1 \;\text{and}\; 2$) should receive the highest relevance score.

The figure~\ref{fig:biased_mkl_scores} illustrates that in the absence of iterative batch normalization (IterBN), our methods exhibit biased scores towards the modality interaction among modalities $1, 2, \text{and}; 3$ (depicted by the orange bars). Notably, the figure demonstrates that IterBN successfully mitigates this bias, correctly assigning the highest relevance score to the modality interaction between modalities $1\;\text{and}\;2$.

\begin{figure}
	\centering
	\includegraphics[width=0.9\columnwidth]{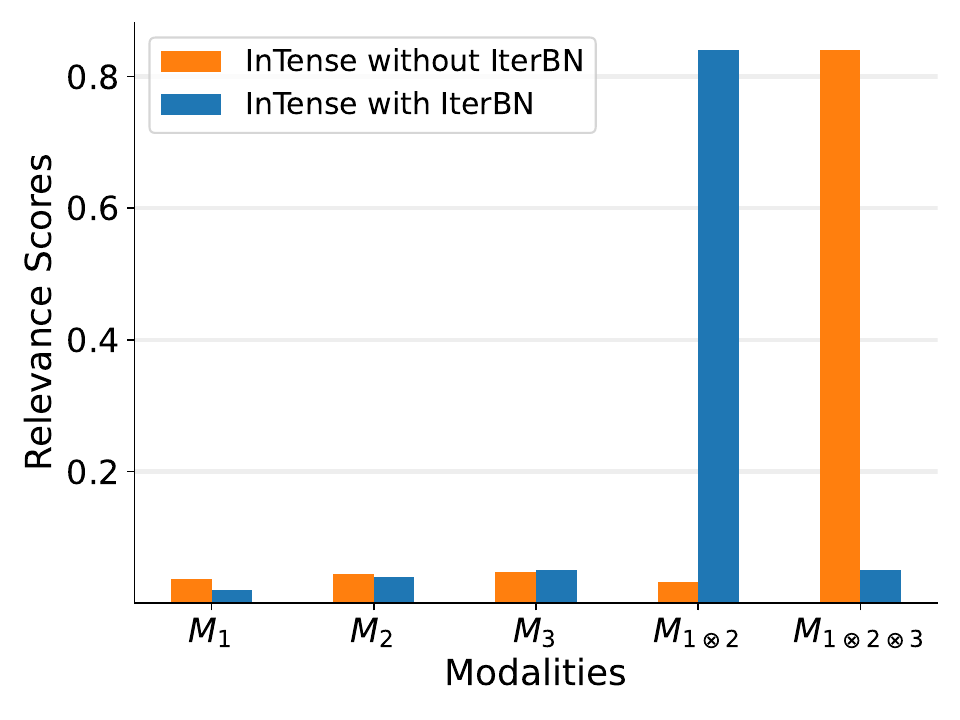}
	\caption{An illustration of our method when higher-order modality-interactions are involved. Naive normalization with VBN leads to biased results (orange bars) where the relevance scores are concentrated toward higher-order interactions $M_{1\otimes2\otimes3}$. In contrast, InTense (blue bars) correctly assigns a high relevance score only to the interaction $M_{1\otimes2}$, which contains all class-specific signals.}
	\label{fig:biased_mkl_scores}
\end{figure}

\subsection{Modality Encoder Modules} 
We tried different encoder modules on the top of the preprocessed features for the approaches MNL, and InTense.  The table \ref{tab:ablation} shows the accuracy obtained on the test set.
\paragraph{GRUwithBN.} These encoder modules are used to compare and report the results in the main paper. The archtecture of the encoder modules is explained in Appendix \ref{sec:encoder}
\paragraph{Tranformer.} We also comapare our results when a Tranformer encoder is used for each modality. We trained a Transformer from scratch using its PyTorch implementation  with $4$ layers and $5$ heads. Table \ref{tab:ablation} shows that these encoder performs slightly better on an average than GRUwithBN on two out of four datasets.
\paragraph{MLP.} We used a Multi Layer Perceptron as an encoder for all the three modalities.

\subsection{MultiRoute global interpretabilty scores}
We compare the interpretability scores of our method against the global interpretability scores obtained from MultiRoute~\cite{tsai2020multimodal}. Figure~\ref{fig:mkl_iter_bn} in the main paper depicts the biased interpretability scores from MultiRoute.

In this section, we explain the process of obtaining the interpretability scores from MultiRoute. At first, we used the same process as described in the paper to obtain the relevance scores. 
At first, we trained a multimodal routing based capsule model on our synthetic dataset \textsc{SynthGene-tri}. Next, as desribed in the paper, we computed the average routing coefficient $\overline{r_{ij}}$ for all three modalities and their interactions. We normalized the routing coefficients so that they sum to $1$. This brings them in the same scale as our interpretability scores.

\subsection{Correlation between Relevance scores and Unimodal Accuracy}
In the main paper, we show through our synthetic dataset that the relevance scores from InTense correlate with the accuracies when trained on individual modalities. Table \ref{tab:correlation} shows that the same correlation can be seen on the real-world multimodal datasets when each modality is trained independently.

\clearpage

\section{Source Code for Reproducibility}
We provide an implementation of our methods in PyTorch as a zip file. In order to run the code, unzip the file's content and navigate into the root directory. The code has been tested on \textit{Python 3.10+}, and \textit{PyTorch 1.12+} on Ubuntu 22.10 and Red Hat Enterprise Linux Server. We recommend using and Linux-based operating system and a virtual environment management tool of your choice to install the packages. We provide step-by-step instructions to set up our code through \textit{Conda}.

\subsection{Code Setup}

\paragraph{Step 1:} Create a conda environment for InTense. 
\begin{lstlisting}[language=Python]
conda create -n intense python=3.10
\end{lstlisting}

\paragraph{Step 2:} Install a compatible PyTorch version inside the environment. Check your compatible version at the official webpage: \textit{https://pytorch.org/}
\begin{lstlisting}[language=Python]
conda activate intense
conda install pytorch torchvision pytorch-cuda=11.7 -c pytorch -c nvidia
\end{lstlisting}

\paragraph{Step 3:} Set up the code to install all the required packages by running the following commands. We use \textit{pip} for dependency management:
\begin{lstlisting}[language=Python]
pip install -U -r requirements.txt
pip install -e .
\end{lstlisting}
We are ready to run the experiments now!

\subsection{Running Experiments}
We describe the scripts here to run the experiments for synthetic and real-world datasets.

\paragraph{Synthetic Datasets.} In order to run different experiments, use the following scripts:
\begin{enumerate}
    \item \textsc{SythGene} To Run the experiment for the \textsc{SythGene} dataset with independent modalities, run the following command:
    \begin{lstlisting}[language=Python]
python intense/training_modules/synthetic/syn_mnl.py --scheduler-gamma 0.9 --epochs 20 --experiment-name "synthgene" --batch-size 32 --latent-dim 8 --lr 1e-3 --reg-param 0.05
\end{lstlisting}

\item \textsc{SythGene-tri} Run experiment for modality interactions, using:
\begin{lstlisting}[language=Python]
python intense/training_modules/synthetic/syn_xor.py --scheduler-gamma 0.9 --epochs 20 --experiment-name "synthgene" --batch-size 32 --latent-dim 8 --lr 1e-3 --reg-param 0.05
\end{lstlisting}
\end{enumerate}

\paragraph{MOSI} In order to run different experiments, use the follwing scripts:
\begin{enumerate}
    \item MNL (Multiple Neural Learning)
\begin{lstlisting}[language=Python]
python intense/training_modules/regression/train_mnl.py --scheduler-gamma 1 --epochs 100 --is-packed --experiment-name "mnl_mosi" --hidden-dims 64 64 768 --lr 1e-3 --dataset mosi
\end{lstlisting}

    \item InTense 
\begin{lstlisting}
python intense/training_modules/regression/train_intense.py --scheduler-gamma 1 --epochs 5 --is-packed --experiment-name "mkl_intensfuse_mosi_eval_2" --hidden-dims 64 64 768 --dataset mosi --lr 1e-3 --tf-latent-dim 8 --tf-indices '13'
\end{lstlisting}
    \item MRO
\begin{lstlisting}
python intense/training_modules/regression/train_mro.py --scheduler-gamma 0.9 --epochs 10 --is-packed --experiment-name "mro_mosi_eval" --affine --hidden-dims 64 64 768 --dataset mosi --lr 5e-3 --tf-latent-dim 8
\end{lstlisting}
    \item MulT
\begin{lstlisting}
python intense/training_modules/regression/train_mult.py --scheduler-gamma 0.9 --epochs 100 --experiment-name "mosi_mult" --lr 1e-3 --dataset mosi --num-workers 2
\end{lstlisting}
\end{enumerate}

\paragraph{MOSEI} In order to run different experiments, use the following scripts:
\begin{enumerate}
    \item MNL (Multiple Neural Learning)
\begin{lstlisting}[language=Python]
python intense/training_modules/regression/train_mnl.py --scheduler-gamma 0.9 --epochs 100 --is-packed --experiment-name "mnl_mosei" --hidden-dims 64 64 1024 --lr 1e-3 --dataset mosei_senti
\end{lstlisting}

    \item InTense 
\begin{lstlisting}[language=Python]
python intense/training_modules/regression/train_intense.py --scheduler-gamma 0.9 --epochs 200 --is-packed --experiment-name "intense_mosei" --hidden-dims 64 64 1024 --dataset mosei_senti --lr 1e-3 --tf-latent-dim 8 --tf-indices '13'
\end{lstlisting}

    \item MRO
\begin{lstlisting}
python intense/training_modules/regression/train_mro.py --scheduler-gamma 0.9 --epochs 100 --is-packed --experiment-name "mro_mosei" --affine --hidden-dims 32 32 128 --dataset mosei_senti --lr 1e-3 --tf-latent-dim 8 --num-workers 4
\end{lstlisting}
    \item MulT
\begin{lstlisting}
 python intense/training_modules/regression/train_mult.py --scheduler-gamma 0.9 --epochs 5 --experiment-name "mult_mosei" --lr 1e-3 --dataset mosei_senti --num-workers 2
\end{lstlisting}
\end{enumerate}

\paragraph{UR-FUNNY} In order to run different experiments, use the following scripts:
\begin{enumerate}
    \item MNL (Multiple Neural Learning)
\begin{lstlisting}[language=Python]
python intense/training_modules/classification/train_mnl.py --scheduler-gamma 1 --epochs 100 --is-packed --experiment-name "humor_mkl" --hidden-dims 32 32 128 --lr 1e-3 --reg-param 0.01 --batch-size 16 --num-workers 4 --dataset humor
\end{lstlisting}

    \item InTense 
\begin{lstlisting}[language=Python]
python intense/training_modules/classification/train_intense.py --scheduler-gamma 0.9 --epochs 20 --is-packed --experiment-name "humor_intense" --hidden-dims 32 32 64 --tf-latent-dim 4 --lr 1e-3 --scheduler-gamma 1 --reg-param 0.01 --batch-size 32 --num-workers 2 --dataset humor --tf-indices '23'
\end{lstlisting}
    \item MRO
\begin{lstlisting}[language=Python]
python intense/training_modules/classification/train_mro.py --scheduler-gamma 0.9 --epochs 200 --is-packed --experiment-name "humor_mro_eval_2" --hidden-dims 32 32 256 --lr 1e-3 --scheduler-gamma 0.9 --reg-param 0.01 --batch-size 16 --num-workers 2 --tf-latent-dim 16 --dataset humor
\end{lstlisting}
    \item MulT
\begin{lstlisting}[language=Python]
python intense/training_modules/classification/train_mnl.py --scheduler-gamma 1 --epochs 100 --is-packed --experiment-name "humor_mkl" --hidden-dims 32 32 128 --lr 1e-3 --reg-param 0.01 --batch-size 16 --num-workers 4 --dataset humor
\end{lstlisting}

\end{enumerate}

\paragraph{MUStARD} In order to run different experiments, use the following scripts:
\begin{enumerate}
    \item MNL (Multiple Neural Learning)
\begin{lstlisting}[language=Python]
python intense/training_modules/classification/train_mnl.py --scheduler-gamma 0.9 --epochs 100 --is-packed --experiment-name "sarcasm_mnl" --hidden-dims 32 32 128 --lr 5e-3 --scheduler-gamma 0.9 --reg-param 0.01 --batch-size 16 --num-workers 4 --dataset sarcasm
\end{lstlisting}

    \item InTense 
\begin{lstlisting}[language=Python]
python intense/training_modules/classification/train_intense.py --scheduler-gamma 0.9 --epochs 200 --is-packed --experiment-name "sarcasm_mkl_intensfuse_eval" --hidden-dims 32 32 128 --tf-latent-dim 8 --lr 1e-3 --scheduler-gamma 0.9 --reg-param 0.01 --batch-size 16 --num-workers 2 --dataset sarcasm --tf-indices '13'
\end{lstlisting}
    \item MRO
\begin{lstlisting}[language=Python]
intense/training_modules/classification/train_mro.py --scheduler-gamma 0.9 --epochs 200 --is-packed --experiment-name "sarcasm_mro_eval_3" --hidden-dims 32 32 256 --lr 1e-3 --scheduler-gamma 0.9 --reg-param 0.01 --batch-size 16 --num-workers 2 --tf-latent-dim 8 --dataset sarcasm
\end{lstlisting}
    \item MulT
\begin{lstlisting}[language=Python]
python intense/training_modules/classification/train_mnl.py --scheduler-gamma 1 --epochs 100 --is-packed --experiment-name "humor_mkl" --hidden-dims 32 32 128 --lr 1e-3 --reg-param 0.01 --batch-size 16 --num-workers 4 --dataset humor
\end{lstlisting}
    
\end{enumerate}

\paragraph{}

All information about the training arguments is provided in Appendix \ref{sec:hparams}. For further information, use the "-\,-help" argument while running the command.

\paragraph{Note.} We use the MLFlow library for experiment tracking; you can run the following command to start the mlflow server from the shell. The relevance scores are automatically computed and tracked during the training.
\begin{lstlisting}
mlflow server --backend-store-uri file:experiments --no-serve-artifacts
\end{lstlisting}

\end{document}